\documentclass[Journal]{IEEEtran}

\usepackage{epsfig}
\usepackage[ruled]{algorithm2e}
\usepackage{url}
\usepackage{subfigure}

\usepackage[stable]{footmisc}
\usepackage{amsmath,amsfonts,mathrsfs,bm,amssymb}
\usepackage{algorithmic}
\usepackage[breaklinks=true]{hyperref}
\usepackage{relsize}
\usepackage{lastpage}
\usepackage[noabbrev]{cleveref}
\crefname{figure}{Figure}{Figures}

\usepackage{booktabs}       
\usepackage{amsfonts}       
\usepackage{nicefrac}       
\usepackage{microtype}      
\RequirePackage[dvipsnames]{xcolor}
\usepackage[breakable,skins]{tcolorbox}
\usepackage{color, colortbl}
\definecolor{LightCyan}{rgb}{0.88,1,1}
\usepackage{soul}
\usepackage{graphicx}
\usepackage{epstopdf}
\usepackage{xcolor}
\usepackage{caption}
\usepackage{algorithmic}

\newtheorem{theorem}{Theorem}

\numberwithin{equation}{section}
\newtheorem{prop}{Proposition}
\newtheorem{lem}{Lemma}
\newtheorem{exmp}{Example}

\input{macros.tex}

\begin{document}

\title{Max-Linear Regression by Convex Programming}

\author{Seonho Kim,~\IEEEmembership{Student Member,~IEEE,}
        Sohail Bahmani, 
        and~Kiryung Lee,~\IEEEmembership{Senior Member,~IEEE}
\thanks{S.K. and K.L. are with the Department of Electrical and
Computer Engineering, The Ohio State University, Columbus, OH 43220 USA (e-mail: kim.7604@osu.edu; kiryung@ece.osu.edu).}
\thanks{S.B. is with the School of Electrical and Computer Engineering, Georgia Institute of Technology, Atlanta, GA 30332 USA (e-mail:sohail.bahmani@ece.gatech.edu).}}


\maketitle

\begin{abstract}
We consider the multivariate max-linear regression problem where the model parameters $\boldsymbol{\beta}_{1},\dotsc,\boldsymbol{\beta}_{k}\in\mathbb{R}^{p}$ need to be estimated from $n$ independent samples of the (noisy) observations $y = \max_{1\leq j \leq k} \boldsymbol{\beta}_{j}^{\mathsf{T}} \boldsymbol{x} + \mathrm{noise}$. The max-linear model vastly generalizes the conventional linear model, and it can approximate any convex function to an arbitrary accuracy when the number of linear models $k$ is large enough. However, the inherent nonlinearity of the max-linear model renders the estimation of the regression parameters computationally challenging. Particularly, no estimator based on convex programming is known in the literature. We formulate and analyze a scalable convex program given by anchored regression (AR) as the estimator for the max-linear regression problem. Under the standard Gaussian observation setting, we present a non-asymptotic performance guarantee showing that the convex program recovers the parameters with high probability. When the $k$ linear components are equally likely to achieve the maximum, our result shows a sufficient number of noise-free observations for exact recovery scales as {$k^{4}p$} up to a logarithmic factor. { This sample complexity coincides with that by alternating minimization (Ghosh et al., {2021}). 
Moreover, the same sample complexity applies when the observations are corrupted with arbitrary deterministic noise.
We provide empirical results that show that our method performs as our theoretical result predicts, and is competitive with the alternating minimization algorithm particularly in presence of multiplicative Bernoulli noise. Furthermore, we also show empirically that a recursive application of AR can significantly improve the estimation accuracy.}
\end{abstract}

\begin{IEEEkeywords}
nonlinear regression,  convex programming, max-linear model, empirical processes, and sample complexity
\end{IEEEkeywords}

%
\IEEEpeerreviewmaketitle

\section{Introduction}

\IEEEPARstart{W}{e} consider the problem of estimating the parameters $\mb \beta_{\star,1}, \dotsc, \mb \beta_{\star,k} \in \mbb{R}^p$ that determine the \emph{max-linear} function
\begin{equation}
\label{eq:maxlinear}
\mb x \in \mbb{R}^p \mapsto \max_{j\in [k]}\ \langle \mb \beta_{\star,j}, \mb x\rangle\,,
\end{equation}
from independent and identically distributed (i.i.d.) observations, where $[k]$ denotes the set $\{1,\dots,k\}$. Specifically, given the data points $\mb x_1,\dotsc,\mb x_n \in \mbb{R}^p$, and denoting the value of a max-linear function, with parameter $\mb \beta$, at these points by
\begin{equation}
\label{eq:deffi}
f_i(\mb \beta) := \max_{j\in[k]} \ \langle \mb x_i, \mb \beta_j \rangle \,,
\end{equation}
we observe the nonlinear measurement
\[y_i = f_i(\mb \beta_\star) + w_i\,,\]
of the parameter vector $\mb \beta_\star = [\mb \beta_{\star,1};\,\dotsc\,;\mb \beta_{\star,k}] \in \mbb{R}^{kp}$ where $w_i$ denotes noise for $i \in [n]$.

{The most relevant prior work studied an \textit{alternating minimization} (AM) algorithm to solve a slightly more general problem of max-affine regression \cite{ghosh2021max}.
Each iteration 
consists of a step to identify the maximizing linear models followed by least-squares update of model parameters. }
{However, we observed that their empirical performance significantly degrades with outliers, mainly due to the sensitivity of the ``maximizer identification'' step. 
} 
Leveraging recent theory for convexifying nonlinear inverse problems in the original domain \cite{bahmani2017phase,bahmani2019estimation,bahmani2019solving}, we propose an alternative approach by convex programming. Due to the inherent geometry of the formulation, the convex estimator provides stable performance in the presence of {adversarial} noise. 
It is worth mentioning that Ghosh et al. \cite{ghosh2021max} considered a random noise model, 
whereas we consider a deterministic ``gross error'' model. {Nevertheless, in the noiseless case, both results achieve exact parameter recovery at comparable sample complexities.}

\subsection{Convex estimator}
The common estimators for $\mb \beta_\star$ such as the \emph{least absolute deviation} (LAD), i.e.,
\begin{equation}
\label{eq:LADestimator}
\mathop{\mathrm{minimize}}_{\mb \beta} ~ \frac{1}{n} \sum_{i=1}^n \left|f_i(\mb \beta) - y_i\right|,
\end{equation}
are generally hard to compute as they involve nonconvex optimization. Given an ``anchor vector'' $\mb \theta$, we study the estimation of $\mb \beta_\star$ through \emph{anchored regression} (AR) that formulates the estimation by the convex program
\begin{equation}
\label{eq:estimator}
\def\arraystretch{1.5}
\begin{array}{ll}
\displaystyle \mathop{\mathrm{maximize}}_{\mb \beta} & \langle \mb \theta, \mb \beta \rangle \\
\mathrm{subject~to} & \displaystyle \frac{1}{n}\sum_{i=1}^n \left(f_i(\mb \beta) - y_i\right)_+ \leq \eta,
\end{array}
\end{equation}
where $(\cdot)_+$ denotes the positive-part function. 
{The parameter $\eta$ should be chosen so that the feasible set of \eqref{eq:estimator} is not empty.}
The anchored regression can be interpreted as a convexification of the LAD estimator. Since the observation functions \eqref{eq:deffi} are convex,  the LAD is nonconvex mainly due to the effect of the absolute value operator in \eqref{eq:LADestimator}. This source of nonconvexity is removed in anchored regression by relaxing the absolute deviation to the positive part of the error. The linear objective that is determined by the anchor vector $\mb \theta$ acts as a ``regularizer'' to prevent degenerate solutions and guarantees exact recovery of the true parameter $\mb \beta_\star$ under certain conditions on the measurement model in the noiseless scenario.

Anchored regression has been originally developed as a scalable convex program to solve the phase retrieval problem \cite{bahmani2017phase,goldstein2018phasemax} with provable guarantees. Anchored regression is highly scalable compared to other convex relaxations in this context \cite{candes2013phaselift,waldspurger2015phase} that rely on semidefinite programming. The idea of anchored regression is further studied in a broader class of nonlinear parametric regression problems with convex observations \cite{bahmani2019solving} and \emph{difference of convex} functions \cite{bahmani2019estimation}.

\subsection{Background and motivation}

A closely related problem is the max-affine regression problem. A max-affine model generalizes the max-linear model in \eqref{eq:maxlinear} by introducing an extra offset parameter to each component. Alternatively, by fixing any regressor to constant $1$, each linear component in \eqref{eq:maxlinear} has its range away from the origin, which turns the model into a max-affine model. Thus methods developed for the two models are compatible. For the sake of simplicity, we use the description in \eqref{eq:maxlinear}. If necessary, a coordinate of $\mb x_i$ can be fixed to $1$ for all $i \in [n]$. 

Since the max-affine model can approximate a convex function to an arbitrary accuracy, it has been utilized in numerous applications, particularly in machine learning and optimization.
Recently it has been shown that an extension called the \emph{max-affine spline operators} (MASOs) can represent a large class of \emph{deep neural networks} (DNNs) \cite{balestriero2018mad,balestriero2019geometry,
balestriero2020max}. They leveraged the model to analyze the expressive power of various DNNs. Max-affine model has also been leveraged to approximate Bregeman divergences in metric learning \cite{siahkamari2019learning} and utility functions in energy storage and beer brewery optimization problems \cite{balazs2016convex}.

As mentioned above, the max-affine and max-linear regression problems are challenging due to the inherent nonlinearity in the model. In the literature, the max-affine regression problem has been studied mostly as a nonlinear least squares \cite{magnani2009convex,toriello2012fitting,hannah2013multivariate,balazs2016convex,ghosh2021max,ho2019dca}:
\begin{equation}
\label{eq:LSest}
\hat{\mb \beta}=\argmin_{[\mb \beta_1; \dots; \mb \beta_k]}\sum_{i=1}^{n}\left(\max_{1\leq j \leq k}\langle\mb x_i,\mb \beta_j\rangle-\mb y_i\right)^2.
\end{equation} 
By utilizing a special structure in \eqref{eq:LSest}, a suite of {iterative optimization algorithms} have been developed \cite{magnani2009convex,toriello2012fitting,hannah2013multivariate}. 
The fact that \eqref{eq:maxlinear} is a special case of piecewise linear function allows us to divide $\mb x_1,\ldots, \mb x_n$ into $k$ partitions based on their membership in the polyhedral cones
\begin{equation}
\label{def:plyhdrl}
\mc C_j := \left\{ \mb w \in \mbb{R}^p ~:~ \left\langle \mb w, \mb \beta_{\star,j} - \mb \beta_{\star,l} \right\rangle \geq 0, ~ \forall l \neq j \right\}, \quad j \in [k],
\end{equation}
which are pairwise \emph{almost disjoint}\footnote{We say two sets are almost disjoint whenever their intersection has measure zero with respect to an underlying measure.} and cover the entire space $\mbb{R}^p$.
In other words, $\mc C_j$ is determined according to which component achieves the maximum in the max-linear model in \eqref{eq:maxlinear}. 
If this oracle information is known a priori, then the estimation is divided into $k$ decoupled linear least squares given by
\begin{equation}
\label{eq:decoupled_leastsquared}
\hat{\mb \beta}_j = \argmin_{\mb \beta_j} \sum_{i \in \mc C_j} \left(\langle\mb x_i,\mb \beta_j\rangle-\mb y_i\right)^2, \quad j \in [k].
\end{equation} 
{
However, since the oracle partition information is not available in practice, various adaptive partitioning methods have been studied. 
Magnani and Boyd \cite{magnani2009convex} proposed the \emph{least-squares partition algorithm}, which is an alternating minimization algorithm that progressively refines the estimates for both model parameters and partitions similar to the $k$-means clustering algorithm. 
Hannah and Dunson \cite{hannah2013multivariate} proposed the \emph{convex adaptive partitioning} (CAP) method, which is a greedy algorithm that builds a partitioning of covariates through dyadic splitting and refit. 
They have shown that the CAP method is asymptotically consistent. 
Bal\'asz \cite{balazs2016convex} proposed the \emph{adaptive max-affine partitioning algorithm}, which combines AM and CAP using a cross-validation scheme and significantly reduces computation time by partitioning at the median. 
The performance of all these algorithms critically depends on the initialization. 
Moreover, they proposed an initialization scheme based on a random search, but its search space grows exponentially in $p$. 
In a later work, Ghosh et al. \cite{ghosh2021max} further improved the random search method by using a spectral method so that the size of the search space does not depend on $p$, even though it grows exponentially in $k$. 
Their initialization scheme remains a practical method when $k = O(1)$. 
Toriello and Vielma \cite{toriello2012fitting} formulated \eqref{eq:LSest} as a mixed integer program based on the ``big-M'' method. { Similar to other methods based on mixed integer programming, their method also suffers from a high computational cost and does not scale well to large instances.} 
Ho et al. \cite{ho2019dca} applied the \emph{DC algorithm} \cite{tao1998dc} to a reformulation of \eqref{eq:LSest} as a \emph{difference-of-convex} program. They showed, empirically, fast convergence of their algorithm to a local minimum.}

{ The aforementioned methods demonstrated satisfactory empirical performance on selected benchmark sets at a tractable computational cost. However, except for the method of Ghosh et al. \cite{ghosh2021max}, these methods lack non-asymptotic statistical guarantees even under reasonable simplifying assumptions. 
A non-asymptotic analysis for the alternating minimization method is first established in \cite{ghosh2021max} which also provides a provably accurate initialization scheme.}

\subsection{Contributions}
We provide a scalable convex estimator for the max-linear regression problem that is formulated as a linear program and is backed by statistical guarantees.
Under the standard Gaussian covariate model, the convex estimator \eqref{eq:estimator} is guaranteed to recover the regression parameters exactly with high probability if the number of observations $n$ scales as $\pi_{\min}^{-4} p$ up to some logarithmic factors { where $\pi_{\min}$ is defined as $\min_{j \in [k]} \P\left(\mb g\in\mc C_j\right)$ for $\mb g\in\mathrm{Normal}(\mb 0,\mb I_p)$}. This sample complexity implicitly depends on $k$ (i.e., the number of components) through $\pi_{\min}$.  Particularly, when the $k$ linear components form a ``well-balanced partition'' in the sense that they are equally likely to achieve the maximum, the smallest probability $\pi_{\min}$ is close to $1/k$ and the derived sample complexity reduces to $k^4 p$ up to the logarithmic factors. { This is comparable to the sufficient condition for exact recovery $n=\mathcal{O}(k p \pi_{\min}^{-3})$ of alternating minimization algorithm {\cite{ghosh2021max}} in the noise-free scenario. Monte Carlo simulations show that our proposed convex estimator, as a convexification of the LAD estimator, exhibits robustness against outliers, whereas AM appears to be fragile in the presence of impulsive noise.  Furthermore, the repetition of AR significantly improves the accuracy of the estimation.}

\section{Accuracy of the Convex Estimator}
\label{sec:mainresult}
In this section, we provide our main results on the estimation error of { the convex program} in \eqref{eq:estimator}. 
We consider {the} anchor vector $\mb \theta$ constructed from a given initial estimate $\tilde{\mb \beta} = [\tilde{\mb \beta}_1;\dots;\tilde{\mb \beta}_k] \in \mbb{R}^{kp}$ as
{
\begin{equation}
\label{eq:anchorvector}
\mb \theta = \frac{1}{2n} \sum_{i=1}^n \nabla f_i(\tilde{\mb \beta})=\frac{1}{2n}\sum_{i=1}^n\sum_{j=1}^k\bbone_{\{\mb x_i\in \tilde{\mathcal{C}}_j\}}\mb e_j\otimes\mb x_i,
\end{equation}
where 
\begin{equation}
\label{def:plyhdrl_tilde}
\tilde{\mc C}_j := \left\{ \mb w \in \mbb{R}^p ~:~ \langle \mb w, \tilde{\mb \beta}_j - \tilde{\mb \beta}_l \rangle \geq 0, ~ \forall l \neq j \right\}, \quad j \in [k]
\end{equation}
and $\mb e_j \in \mathbb{R}^k$ denotes the $j$th column of the $k$-by-$k$ identity matrix $\mb I_k$ for $j \in [k]$.
Since $f_i$ is differentiable except on a set of measure zero, with a slight abuse of terminology, $\nabla f_i$ in \eqref{eq:anchorvector} is referred to as the ``gradient''.} { In \eqref{eq:anchorvector}, the choice of anchor vector follows from the geometry of convex equations \cite[Section~1.4]{bahmani2019solving}. 
In particular, in the noiseless case, $\mb \beta_\star$ would be a solution to 
\[
\def\arraystretch{1.5}
\begin{array}{ll}
\displaystyle \mathop{\mathrm{maximize}}_{\mb \beta} & \langle \mb \theta, \mb \beta \rangle \\
\mathrm{subject~to} & \displaystyle f_i(\mb \beta)\leq y_i,\quad\forall i\in[n].
\end{array}  
\] 
if it satisfies the Karush–Kuhn–Tucker condition
\[
-\mb \theta+\sum_{i=1}^{n}\lambda_i\nabla f_i(\mb \beta_\star)= \mb 0
\] for some $\lambda_1,\ldots,\lambda_n \geq 0$. 
In other words, the anchor vector $\mb \theta$ needs to be in the $\mathrm{cone}\left(\left\{\nabla f_i(\mb \beta)\right\}_{i=1}^n\right)$.
The choice of $\mb \theta$ in \eqref{eq:anchorvector} is inspired by this condition.} 

The following theorem illustrates the sample complexity and the corresponding estimation error achieved by the estimator in \eqref{eq:estimator}. The estimation error is measured as the sum of the $\ell_2$ norms of the difference between the corresponding components of the ground truth $\mb \beta_\star$ and the estimate $\hat{\mb \beta}$.

\begin{theorem}
\label{thm:main}
Let $\{\mc C_j\}_{j=1}^k$ and $\{\tilde{\mc C}_j\}_{j=1}^k$ be respectively defined as in \eqref{def:plyhdrl} and \eqref{def:plyhdrl_tilde}. 
Let $\mb \theta$ be as in \eqref{eq:anchorvector} and $\{\mb x_i\}_{i=1}^n$ be independent copies of $\mb g \sim \mathrm{Normal}(\mb 0, \mb I_p)$. 
Then there exist absolute constants $c,C>0$, for which the following statement holds for all $\mb w \in \mathbb{R}^n$ with probability at least $1 - \delta$: 
Suppose that $\tilde{\mb \beta}$ is independent of $\{\mb x_i\}_{i=1}^n$ satisfies
\begin{equation}
\label{eq:initial_condition}
\begin{aligned}
&\frac{\|(\tilde{\mb \beta}_j-\tilde{\mb \beta}_{j'}) - (\mb \beta_{\star,j}-\mb \beta_{\star,j'})\|_2}{\|\mb \beta_{\star,j}-\mb \beta_{\star,j'}\|_2} \leq\\
& \min\left(0.1,
\frac{c \pi_{\min}^4}{2k}\log^{-1/2}\left(\frac{k}{c\pi_{\min}^4}\right) \right),\,~\forall j, j' \in [k]:j \neq  j'.
\end{aligned}
\end{equation} If the feasible set of the optimization problem in \eqref{eq:estimator} is not empty and the number of observations satisfies
\begin{equation}
\label{thm:samplecom_noise}
n \geq C \, {\zeta}^{-2} \left( 4p \log^3p \log^5k + 4\log(1/\delta) \log k \right),
\end{equation} 
where
\[
\zeta := {\min_{j\in[k]}
\sqrt{\frac{\pi}{32}} \, \P^2\{\mb g\in {\mc C_j}\}}-{2\max_{j\in[k]} \sqrt{\P\{\mb g\in\tilde{\mc C}_j \triangle{ \mc C_j}\}}},
\] 
then the solution $\hat{\mb \beta}$ to \eqref{eq:estimator} obeys
\begin{equation}
\label{thm:err_bnd}
\sum_{j=1}^{k}\|\mb \beta_{\star,j}-\hat{\mb \beta}_j\|_2\leq\frac{2}{\zeta}\left(\eta+\frac{1}{n}\sum_{i=1}^n (w_i)_+\right).
\end{equation}
\end{theorem}
{

To make the optimization problem in \eqref{eq:estimator} feasible, it suffices to include the ground-truth $\mb \beta_\star$ in the feasible set, i.e. 
\begin{equation}
\label{eq:cond_eta}
\eta\geq\frac{1}{n}\sum_{i=1}^{n}(-w_i)_+,
\end{equation}}
The error bound in \eqref{thm:err_bnd} reduces to $\frac{2}{\zeta n}\sum_{i=1}^{n}|w_i|$ when the parameter $\eta$ is chosen so that the equality in \eqref{eq:cond_eta} is achieved. 
In practice, the noise entries are unknown and this error cannot be achieved. 
If $\eta$, as a parameter that determines the power of the adversary, is chosen so that $\eta \geq \|\mb w\|_1/n$, then the resulting error bound becomes $\frac{4 \|\mb w\|_1}{n\zeta}$. 
In particular, if $\eta$ satisfies $\eta \geq \|\mb w\|_\infty$, then the resulting error bound will be $\frac{4 \|\mb w\|_\infty}{\zeta}$. 
The latter condition will be readily satisfied in practical applications. { Furthermore, as shown in the empirical sensitivity analysis in Section~\ref{sec:experiment}, the estimation error does not crucially depend on the choice of $\eta$.}

{
\subsection{Comparison with an oracle estimator}
\label{sec:compare_oracle}

Assuming that the additive noise is i.i.d. sub-Gaussian with zero mean and variance $\sigma^2$, the error bound in \eqref{thm:err_bnd} becomes $\tilde{O}(\sigma/\zeta)$, which implies that our estimator is not consistent. 
However, in the adversarial noise setting which is our focus, we can compare the performance case of our estimator with an oracle-assisted estimator, similar to the analysis carried out in \cite{candes2010matrix} for the matrix completion problem.
In this scenario, the error bound by the convex estimator nearly matches the performance of an oracle-assisted estimator (up to a factor determined by $\mb \beta_\star$).

\begin{lem} 
\label{lem:oracle_bound}
Consider the same regression problem as in \Cref{thm:main} with $\{\mb x_i\}_{i=1}^n$ being independent copies of $\mb g \sim \mathrm{Normal}(\mb 0, \mb I_p)$. 
Suppose that $\{\mathcal{C}_j\}_{j=1}^k$ in \eqref{def:plyhdrl} is given as the oracle information. Then there exists an absolute constant $C > 0$ such that if 
\begin{equation}
\label{eq:sample_comp_oracle}
n \geq C \pi_{\min}^{-2} \max(kp\log(n/p),\log(1/\delta)),
\end{equation}
then the estimates $\{\hat{\mb \beta}_j\}_{j=1}^k$ obtained through the decoupled least-squares \eqref{eq:decoupled_leastsquared} satisfy
\begin{equation}
\label{eq:lb_oracle}
{\sup_{\|\mb w\|_\infty \leq \eta'}}
\sum_{j=1}^k\|\mb \beta_{\star,j}-\hat{\mb \beta}_j\|_2 \gtrsim \frac{\pi_{\min}^{3/2} \eta'}{\pi_{\max}}
\end{equation}
with probability at least $1-\delta$, where $\pi_{\max} := \max_{j \in [k]} \P(\mb g \in \mc C_j)$.
\end{lem}
\begin{IEEEproof}
See \Cref{sec:proof_oracle}.
\end{IEEEproof}

One expects that the oracle estimator nearly achieves the optimal performance. 
However, {since the lower bound by \Cref{lem:oracle_bound} does not vanish as $n$ increases to infinity}, the oracle estimator is also biased in the presence of adversarial noise.  
{Note that the lower bound in \eqref{eq:lb_oracle} remains the same with the feasible set substituted by $\|\mb w\|_1 \leq n \eta'$. 
Furthermore, if $\eta$ achieves the equality in \eqref{eq:cond_eta}, then the error bound in \eqref{thm:err_bnd} implies
\begin{equation}
\label{thm:unif_err_bnd}
\sup_{\|\mb w\|_1 \leq n \eta'} 
\sum_{j=1}^{k}\|\mb \beta_{\star,j}-\hat{\mb \beta}_j\|_2\leq\frac{2 \eta'}{\zeta}.
\end{equation}
Therefore, in this scenario, the error bound in \eqref{thm:unif_err_bnd}} matches that by the oracle estimator up to an extra factor {$O(\pi_{\max}/\zeta \pi_{\min}^{3/2})$}. 
In particular, {if $\pi_{\max} \approx \pi_{\min} \approx 1/k$}, then the error by the convex estimator is sub-optimal up to a factor {$k^{5/2}$} relative to the oracle estimator. }

{
\subsection{Initialization}
\label{subsec:init}
{ Theorem~\ref{thm:main} provides an error bound by the convex estimator given an initial estimate satisfying \eqref{eq:initial_condition}. 
Finding such an initial estimate is not a trivial task. 
Ghosh et al. \cite{ghosh2021max} proposed an initialization scheme that consists of dimensionality reduction by a spectral method \cite[Algorithm~2]{ghosh2021max}, followed by a low-dimensional random search \cite[Algorithm~3]{ghosh2021max}. It has been shown that {if the observations are corrupted with independent sub-Gaussian noise, then} the initialization scheme provides an estimate within a certain neighborhood of the ground-truth in a polynomial time when $k = O(1)$. 
{Their proof only uses the fact that the maximum magnitude of sub-Gaussian noise entries is bounded with high probability. Below, we extend the analysis of their initialization scheme to the scenario where the noise vector $\mb w$ is a fixed \emph{deterministic vector} under the only condition that $\|\mb w\|_\infty \leq \eta'$.}


To this end, we first recall the first stage in their initialization scheme that extracts the eigenvectors corresponding to the $k$ dominant eigenvalues of the following matrix:
\begin{equation}
\label{eq:Mhat}
\widehat{\mb M}=\frac{2}{n} \left( \sum_{i=1}^{n/2}y_i\mb x_i \right) \left( \sum_{i=1}^{n/2}y_i\mb x_i \right)^\top + \frac{2}{n}\sum_{i=1}^{n/2}y_i\left(\mb x_i\mb x_i^\T-\mb I_p\right). 
\end{equation}
{
Let $\widetilde{\mb M}$ denote the noise-free version of $\widehat{\mb M}$, i.e., 
\[
\begin{aligned}
&\widetilde{\mb M} = \frac{2}{n}\sum_{i=1}^{n/2} \left( \max_{j \in [k]} \langle \mb \beta_{\star,j}, \mb x_i \rangle \right) \left(\mb x_i\mb x_i^\T-\mb I_p\right)+\\
&\frac{2}{n} \left( \sum_{i=1}^{n/2} \left( \max_{j \in [k]} \langle \mb \beta_{\star,j}, \mb x_i \rangle \right) \mb x_i \right) \left( \sum_{i=1}^{n/2} \left( \max_{j \in [k]} \langle \mb \beta_{\star,j}, \mb x_i \rangle \right) \mb x_i \right)^\top. 
\end{aligned}
\]
Then the ground-truth parameter vectors $\mb \beta_1^\star, \dots, \mb \beta_k^\star$ are in the columns space of $\E \widetilde{\mb M}$.}
Ghosh et al. {\cite{ghosh2021max}} derived a tail bound on the perturbation of those eigenvectors due to sub-Gaussian noise. 
We provide an analogous perturbation analysis in the { deterministic} noise setting. 
The following lemma provides upper bounds on the contributions of the noise to the two summands in the right-hand side of \eqref{eq:Mhat}. 
\begin{lem}
\label{lem:noise_norm_bound}
Suppose that $\mb x_1,\ldots,\mb x_n \overset{\mathrm{i.i.d.}}{\sim} \mathrm{Normal}(\mb 0,\mb I_p)$ and $\mb w:=\left(w_1,\ldots,w_n\right)\in\mathbb{R}^n$ are arbitrary fixed. 
Then the following inequalities hold with probability at least $1-\delta$:
\begin{equation}
\label{eq:bounds_noise_terms1}
\begin{aligned}
&\left\|\frac{1}{n}\sum_{i=1}^n w_i\mb x_i\right\|_2
\lesssim
\|\mb w\|_\infty \cdot \sqrt{\frac{p + \log(1/\delta)}{n}},\\
&\left\|\frac{1}{n}\sum_{i=1}^n w_i\left(\mb x_i\mb x_i^\T-\mb I_p\right)\right\|
\lesssim\\
&\qquad\qquad\|\mb w\|_\infty \cdot \max\left( \sqrt{\frac{p + \log(1/\delta)}{n}}, \frac{p + \log(1/\delta)}{n} \right).
\end{aligned}
\end{equation} 
\end{lem}
\begin{IEEEproof}
See \Cref{sec:proof:noise_norm_bound}.
\end{IEEEproof}

Let $\hat{\mb U}$ be a matrix whose columns are the $k$ dominant eigenvectors of $\widehat{\mb M}$. 
Furthermore, let the columns of ${\mb U}^\star$ be the eigenvectors of the noise-free component of $\E \tilde{\mb M}$. 
Then, plugging the results in \Cref{lem:noise_norm_bound} into the proof of \cite[Lemma~8]{ghosh2021max} yields that
\begin{equation}
\label{eq:modified_Theorem3}
\begin{aligned}
&\left\|\hat{\mb U}\hat{\mb U}^\T-{\mb U}^\star\left({\mb U}^\star\right)^\T\right\|_{\mathrm{F}}^2\lesssim\\
&\left(\frac{\|\mb w\|_{\infty}^2+\max_{j\in[k]}\|\mb \beta_{\star,j}\|_1^2}{\lambda_k^2(\E {\widetilde{\mb M}})}\right)\frac{kp\log^3(pk/\delta)}{n}
\end{aligned}
\end{equation} 
holds with probability at least $1-\delta$. 
This is analogous to \cite[Theorem~2]{ghosh2021max} which addresses the case of the sub-Gaussian noise. 
The remainder of their initialization scheme does not depend on any assumption on the noise model. 
Therefore, the resulting initial estimate satisfies \eqref{eq:initial_condition} if 
\begin{equation}
\label{eq:sample_init}
\begin{aligned}
&n\gtrsim\frac{k^6\log(k/\pi_{\min})}{\pi_{\min}^{13}}\cdot\\
&\max\Bigg\{{\|\mb w\|_{\infty}^2}\log\left(1+\frac{\max_{j\in[k]}\|\mb \beta_{\star,j}\|_2 k^4\log^{1/2}(k/\pi_{\min})}{\pi_{\min}^{5.5}}\right),\\
&\left(\|\mb w\|_{\infty}^2+\max_{j\in[k]}\|\mb \beta_{\star,j}\|_1^2\right)\frac{{k^{3}} p\log^3(n/k) \cdot\max_{j\in[k]}\|\mb \beta_{\star,j}\|_2}{\lambda_{k}^2(\mathbb{E}{\widetilde{\mb M}})}\Bigg\}.
\end{aligned}
\end{equation}
The condition in \eqref{eq:sample_init} is obtained by applying the initialization condition \eqref{eq:initial_condition} and substituting $\sigma$ by $\|\mb w\|_{\infty}$ in \cite[Equation~20]{ghosh2021max}. 
The requirement for the initial point of anchored regression  \eqref{eq:initial_condition} is more relaxed in terms of the dependence on $\pi_{\min}$, compared to the similar requirement for the alternating minimization method \cite[Theorem~1]{ghosh2021max}. 
Furthermore, for both anchored regression and alternating minimization, the sample complexity of the initialization dominates that of the subsequent stages of the algorithms.}

{In the above paragraphs, we have shown that the anchored regression combined with the spectral initialization provides a stable estimate in the presence of an arbitrarily fixed deterministic noise of bounded magnitudes. However, this result does not extend to the adversarial noise setting in  \Cref{thm:main} and \Cref{lem:oracle_bound}. Maximization over $\mb w$s that obey $\|\mb w\|_\infty \leq \eta'$ in \eqref{eq:bounds_noise_terms1}, can be addressed effectively by { taking the union bound over extreme points of $\ell_\infty^n$ ball with the radius $\eta'$ and choosing $\delta = 2^{-n}\bar{\delta}$ with $\bar{\delta}\in [0,1]$ denoting overall error probability. Therefore, the terms $\frac{p+\log(1/\delta)}{n}$ in \eqref{eq:bounds_noise_terms1} are equal to $\frac{p+n\log(2)+\log(1/\bar{\delta})}{n}$, which are clearly bounded from below by $\log 2$.}
Consequently, in the adversarial setting, the error in the spectral method does not vanish as $n$ grows, and the desired accuracy for the initialization scheme cannot be established. Considering a relaxed condition $\|\mb w\|_1 \leq n \eta'$ exacerbates the situation and the error bound in the spectral method becomes even larger. 
}

{
\subsection{Compariosn with alternating minimization in computational cost}
\label{subsec:comparison}
This section compares AR and AM in their computational costs. First,} AR is implemented via an equivalent formulation with auxiliary variables $\mb t:=[t_1; \dots ; t_n] \in\mathbb{R}^n$ as 

\begin{equation}
\label{eq:estimatorasLP}
\arraycolsep=1.4pt\def\arraystretch{1.5}
\begin{array}{cl}
\displaystyle 
\mathop{\mathrm{maximize}}_{(\mb \beta_j)_{j=1}^k, (t_i)_{i=1}^n} & \displaystyle \langle {\mb \theta}, [\mb \beta_1;\dots;\mb \beta_k] \rangle \\ 
\mathrm{subject~to} & \displaystyle t_i \geq 0,~ \langle \mb x_i, \mb \beta_j \rangle - y_i \leq t_i,~  \frac{1}{n}\sum_{i=1}^{n}t_i\leq\eta, \\
&~~~~~~~~~~~~~~~~~~~~\quad \forall i\in[n],~\forall j\in[k]\,.
\end{array}
\end{equation}
}
{ To compute the computational costs for \eqref{eq:estimatorasLP}, we further reformulate it into the form of a linear program  $\min_{\mb A\mb s=\mb b,\mb s\geq \mb 0} \langle\mb c,\mb s\rangle$ by introducing an additional $nk+1$ auxiliary variables to convert the second and third inequality constraints into equality constraints. Then, we have $nk+1$ equality constraints and $2pk+nk+n+1$ variables. By \cite{van2020deterministic}, finding its exact solution costs $\tilde{O}\left(((n+p)k)^c\right)$ with $c\approx2.38$. In contrast, with finitely many operations, AM can find only an approximate solution. 
The per-iteration cost of AM is ${O}(nkp^2)$. In the noiseless case, due to the linear convergence of AM, the total cost to obtain an $\epsilon$-accurate solution is ${O}(nkp^2\log(1/\epsilon))$. 

In a special case where the observations are almost equally distributed over the linear components of the max-linear model, we have $\pi_{\min} \approx \pi_{\max} \approx 1/k$. 
Consequently, the sample complexity for both estimators is $\tilde{O}(pk^4)$. Thus, the computational costs for AR and AM become $\tilde{O}(p^{2.38}k^{12})$ and $\tilde{O}(p^3k^5)$, respectively. When $p$ is much larger than $k$ (specifically, $p>k^{14}$, the computational cost of AR is significantly lower than that of AM. However, in the opposite scenario, AM is more cost-effective. We summarize the comparison with respect to the computational cost, sample complexity and model assumption in 
\Cref{table:comparison}.
}

\begin{table*}[h]
\label{table:comparison}
\centering
\caption{Comparison of local convergence of AR and AM.\footnotemark}
\begin{tabular}{c|c|c|c|c|c}
\hline
Method & Cost for $\epsilon$-accuracy &  Cost for an ideal instance & Sample complexity & Covariate model  & Noise model \\ 
\hline\hline
AR & $\tilde{O}\left(\left((n+p)k\right)^{2.38}\right)$ & $\tilde{O}(p^{2.38}k^{12})$ & $\tilde{O}\left(\pi_{\min}^{-4}p\right)$ & Gaussian & Adversarial \\
\hline
AM \cite{ghosh2021max} & $O(nkp^2\log(1/\epsilon))$ & $\tilde{O}(p^3k^5)$ & $O(\pi_{\min}^{-3}kp)$ & Gaussian  & Sub-Gaussian \\
\hline
\end{tabular}
\end{table*}
\footnotetext{The spectral initialization is not included in this comparison. To incorporate the initialization into the analysis, it is necessary to modify the noise model from an adversarial noise model to a gross error model as discussed in \Cref{subsec:init}.}

\section{Numerical results}
\label{sec:experiment}
We present a set of Monte Carlo simulations to evaluate the performance of the estimator by anchored regression numerically. 
The experiments were designed to illustrate the following perspectives on the estimation performance: 
i) The empirical phase transition on exact recovery without noise corroborates Theorem~\ref{thm:main}; 
ii) Further iterations of AR with updated anchor vectors significantly reduce the estimation error; 
iii) AR provides a competitive empirical performance with additive Gaussian noise to AM; 
iv) AR provides a stable estimation in the presence of sparse noise, where the performance of AM significantly deteriorates. { We implement AR by the linear program given in \eqref{eq:estimatorasLP}.} Since \eqref{eq:estimatorasLP} is in the standard form of a linear program, it can be solved efficiently by readily available software such as CPLEX and Gurobi \cite{gurobi}. AR is compared to the version of AM by Ghosh et al. {\cite{ghosh2021max}}. 
For a fair comparison, we let both methods start from the same initial estimate, which will be specified later. 

In the Monte Carlo simulations, the regressors $\mb x_1,\dots,\mb x_n$ are generated as independent copies of a random vector following $\mathrm{Normal}(\mb 0,\mb I_p)$, as assumed in Theorem~\ref{thm:main}. 
For each run, the estimation error is measured up to permutation ambiguity, that is, the error is calculated as the minimum of $\sum_{j=1}^{k}\|\hat{\mb \beta}_{\pi(j)} - \mb \beta_{\star,j}\|_2/\sum_{j=1}^{k}\|\mb \beta_{\star,j}\|_2 $ over all possible permutation $\pi$ over segment indices, where $(\mb \beta_{\star,j})_{j=1}^k$ and $(\hat{\mb \beta}_j)_{j=1}^k$ denote the ground-truth parameters and their estimates, respectively. 

Since both AR and AM algorithms operate provided suitably initialized parameter, it is crucial to obtain an initial estimate, which lands near the ground-truth parameter. 
To this end, throughout the simulations, we apply the heuristic known as the \textit{AM with repeated random initialization} in \cite{balazs2016convex}, summarized as follows: One repeats the following procedure for $q \in [m]$: i) Randomly generate parameters $\mb \beta_{q,1}^{r},\ldots,\mb \beta_{q,k}^{r}\in\mathbb{R}^{p}$. 
ii) Run the AM algorithm from given initial estimates for $I_{\mathrm{init}}$ iterations and obtain estimates $\mb \beta_{q,1}^{o}, \dots, \mb \beta_{q,k}^{o}$. 
Then choose the set of parameters $\mb \beta_{q',1}^{o}, \ldots \mb \beta_{q',k}^{o}$, which achieves the least empirical loss in \eqref{eq:LSest}, i.e. 
\[
q'=\argmin_{q\in[m]}\sum_{i=1}^{n}\left(\max_{1\leq j\leq k}\langle\mb x_i,\mb \beta_{q,j}^o\rangle-y_i\right)^2.
\]
Throughout all simulations, the initialization parameters are set to $m=200$ and $I_{\mathrm{init}}=10$. 
Moreover, the maximum iteration number for the AM algorithm, denoted by $I_{\mathrm{AM}}$, is set to $I_{\mathrm{AM}}=120$.

\begin{figure}[!htb]
        \centering
        \subfigure[AR]{\label{fig:ARa}
		\includegraphics[scale=0.28]{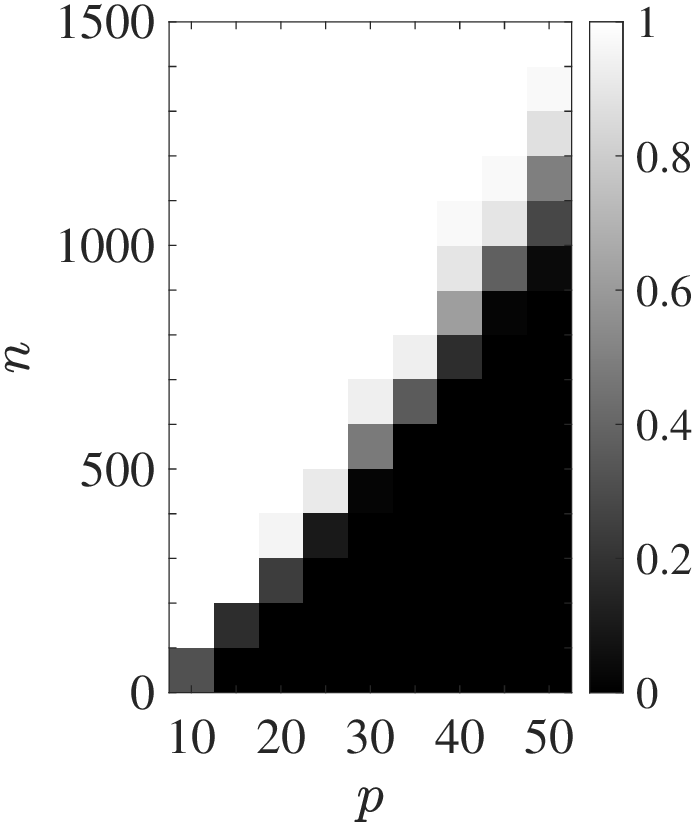}}
	    \hspace{0.05\textwidth}
        \subfigure[AM]{\label{fig:AMa}
	    \includegraphics[scale=0.28]{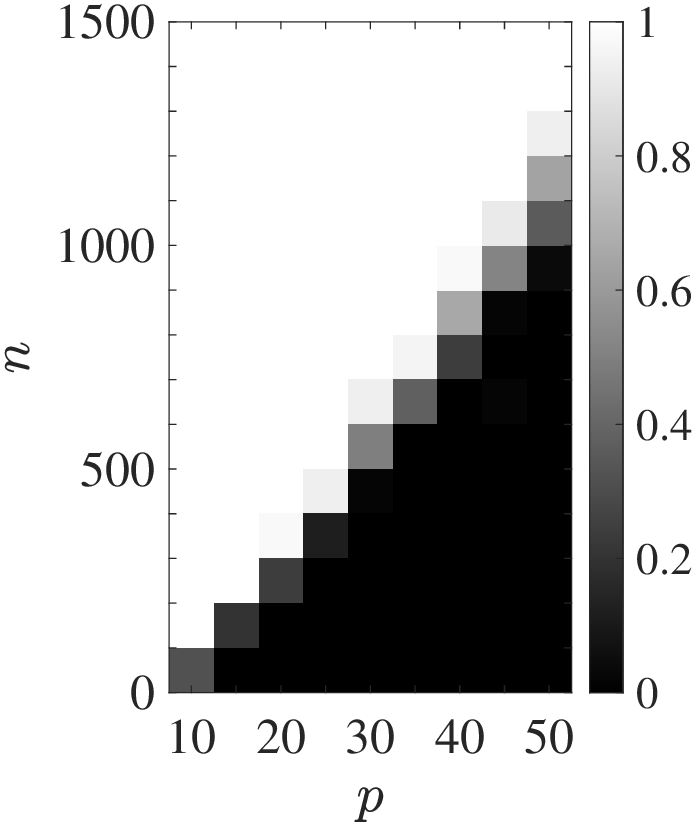}}
        \caption{Phase transition of recovery rate for varying $n$ and $p$ in the noiseless case ($k=5$).}
        \label{fig:k5} 
\end{figure}
\begin{figure}
        \centering
        \subfigure[AR]{
        \label{fig:ARb}
		\includegraphics[scale=0.28]{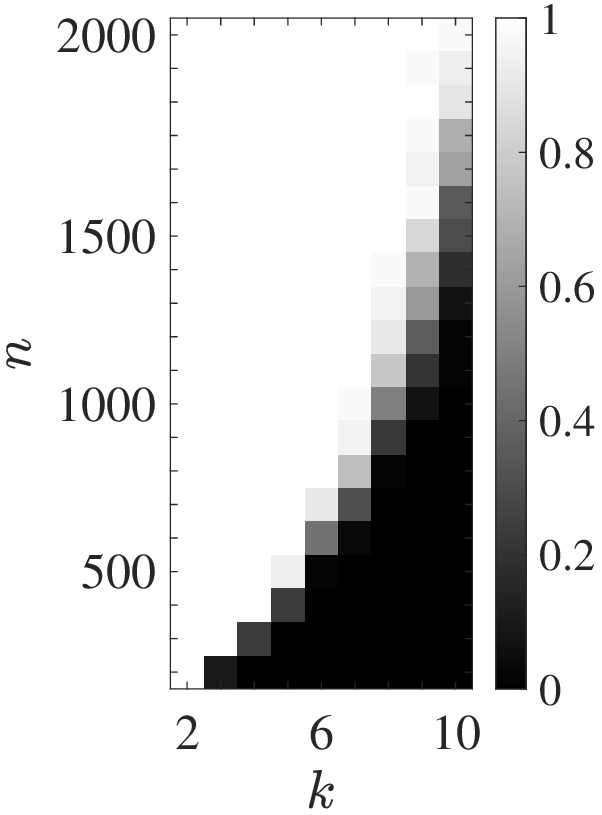}}
	    \hspace{0.05\textwidth}
        \subfigure[AM]{
        \label{fig:AMb}
		\includegraphics[scale=0.28]{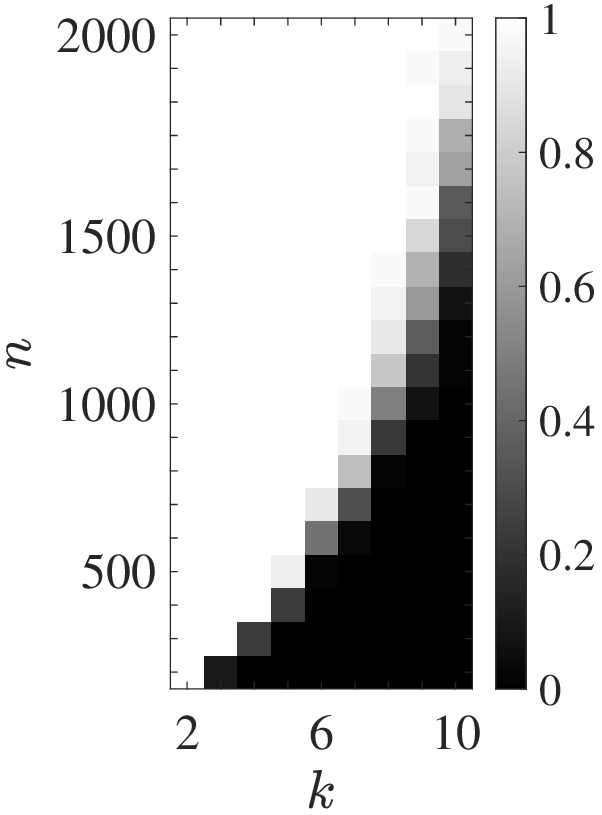}}
        \caption{Phase transition of recovery rate for varying $n$ and $k$ in the noiseless case ($p=20$).}
        \label{fig:p20} 
\end{figure}

\cref{fig:k5,fig:p20} illustrate the empirical phase transition of exact recovery in the noise-free scenario as a function of the sample size $n$ per varying dimension parameters, which are the ambient dimension $p$ and the number of segments $k$.
The reconstruction is determined as success if the normalized estimation error is below $10^{-5}$. 
The recovery rate is calculated as the ratio of success out of $50$ trials. 
In this simulation, we assume that $k \leq p$. 
To satisfy the ``well-balance partition'' condition, we generate the ground-truth parameter vectors so that they are mutually orthogonal one another. 


Figure~\ref{fig:k5} shows that for both AR and AM, the phase transition occurs when $n$ grows linearly with $p$ while $k$ is fixed to $5$. 
This observation qualitatively coincides with the sample complexity by Theorem~\ref{thm:main}.
A complementary view is provided by Figure~\ref{fig:p20} for varying $k$ while $p$ is fixed to $20$.  
Here, the phase transition occurs when $n$ is proportional to $k^t$ for some constant $t\in(1,2)$. 
The order of this polynomial is smaller than the corresponding result by Theorem~\ref{thm:main}, where $n$ is proportional to $k^4$. 
{A similar gap between theoretical sufficient condition and empirical phase transition was observed for AM in the noise-free setting \cite[Appendix~L]{ghosh2019max}. }
Overall, as shown in these figures, AR and AM provide similar empirical performance in the noiseless scenario. 



\begin{algorithm}[t]
	\caption{Iterative Anchored Regression (IAR)}
	\label{ItrAR}
	\begin{algorithmic}[1]
		\STATE {\bfseries Input:} data $\left\{\mb x_i, y_i\right\}_{i=1}^n$; initialized parameter ${\tilde {\mb \beta}}\in{\mathbb{R}^{kp}};$ fidelity upper bound $\eta$; {max. number of iterations} $I_{\mathrm{IAR}}$
		\STATE {\bfseries Output:} estimated parameter ${\hat {\mb \beta}}\in\mathbb{R}^{pk}$
		\FOR{$i=1$ {\bfseries to} $I_{\mathrm{IAR}}$}
 		\STATE{Compute anchor vector $\mb \theta$ from $\tilde{\mb \beta}$ by \eqref{eq:anchorvector}}
		\STATE {
		    Estimate $\hat {\mb \beta}$ by anchored regression in \eqref{eq:estimatorasLP}
		}
		\STATE{
		$\tilde{\mb \beta} \leftarrow \hat{\mb \beta}$}
		\ENDFOR

	\end{algorithmic}
\end{algorithm}

\begin{figure*}
\center
\hfill
\subfigure[$\sigma=0.05$]{
\includegraphics[scale=0.28]{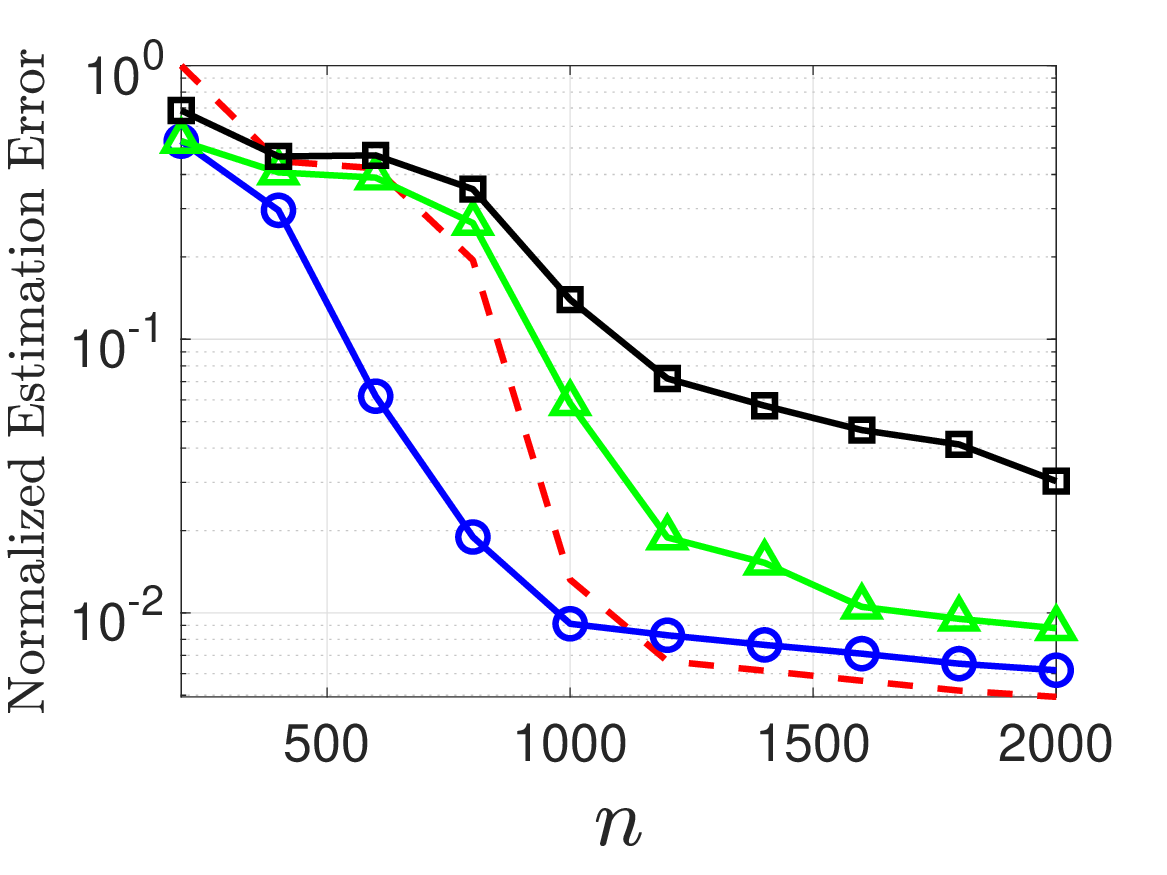}}
\hfill
\subfigure[$\sigma=0.1$]{
\includegraphics[scale=0.28]{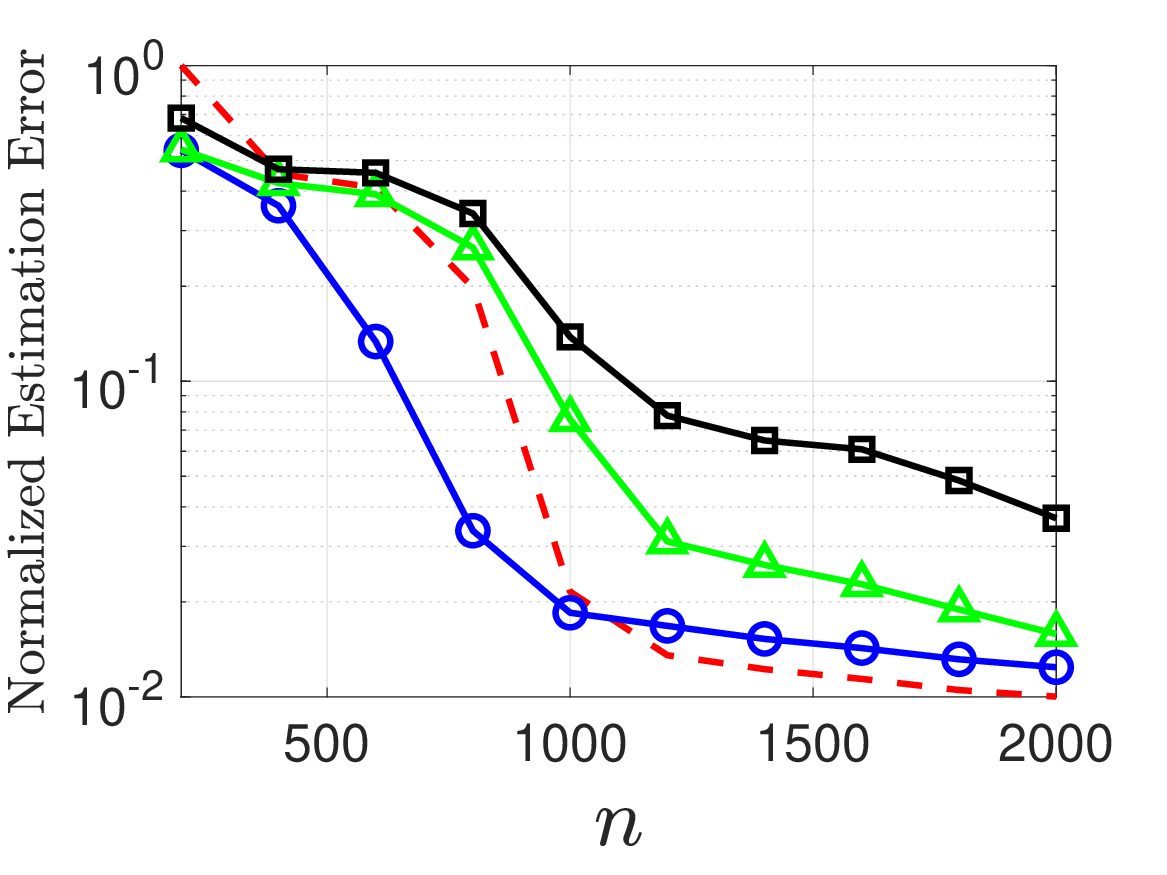}}
\hfill
\subfigure[$\sigma=0.2$]{
\includegraphics[scale=0.28]{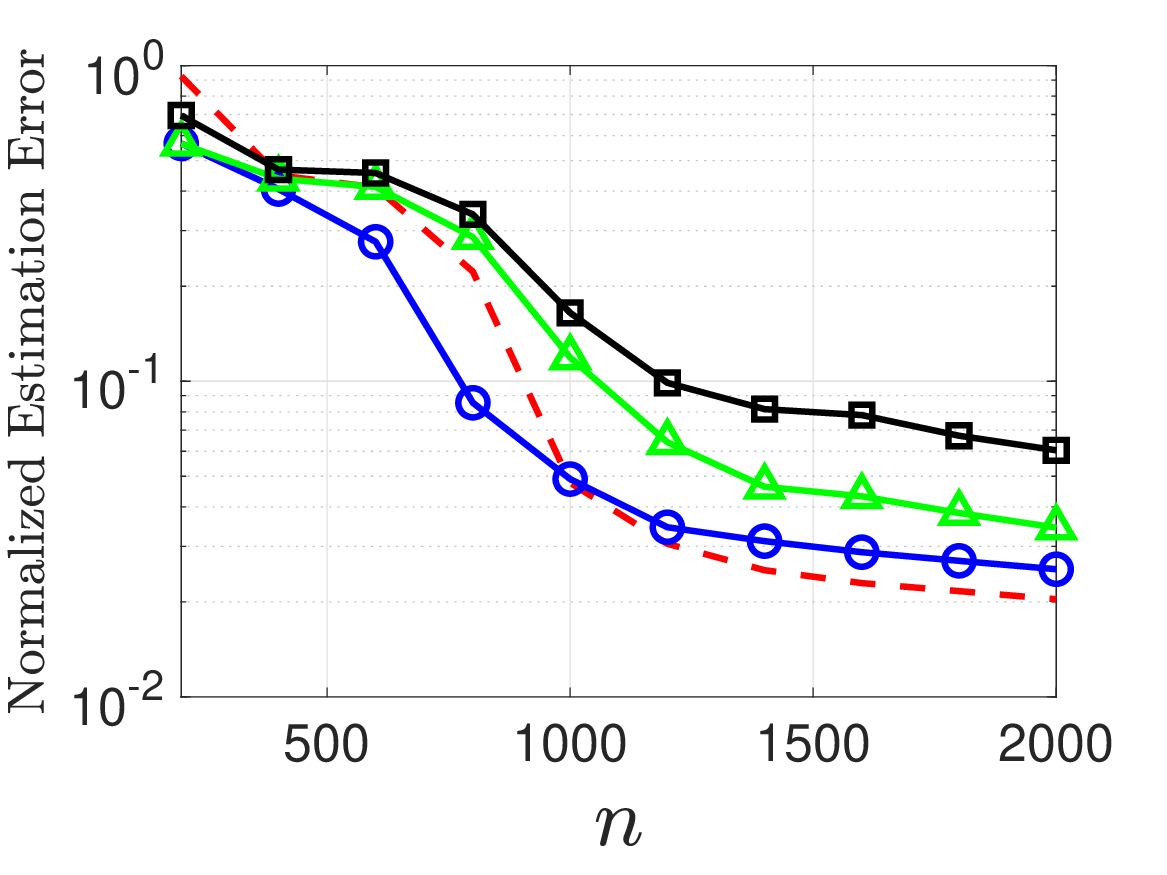}}
\hfill
\caption{Estimation error versus the number of observations $n$ under Gaussian noise of variance $\sigma^2$ ($k=6$ and $p=30$): repeated random initialization (black line with square markers), AR (green line with triangle markers), iterative AR (blue line and circle markers), and AM (red dashed line). All methods start from the repeated random initialization. 
}
\label{fig:gaus_noise}
\end{figure*}

 
In practice, observations are often corrupted with noise. Next, we study the estimation under two noise models. In these experiments, the ground-truth parameter vectors are i.i.d $\mathrm{Normal}(\mb 0,\mb I_{kp})$. Furthermore, to deduce statistical performance, the median of the estimation error in $50$ trials is observed.
{
 \begin{figure}[ht]
    \centering    
    \begin{minipage}[b]{0.4\textwidth}
    \includegraphics[width=\textwidth]{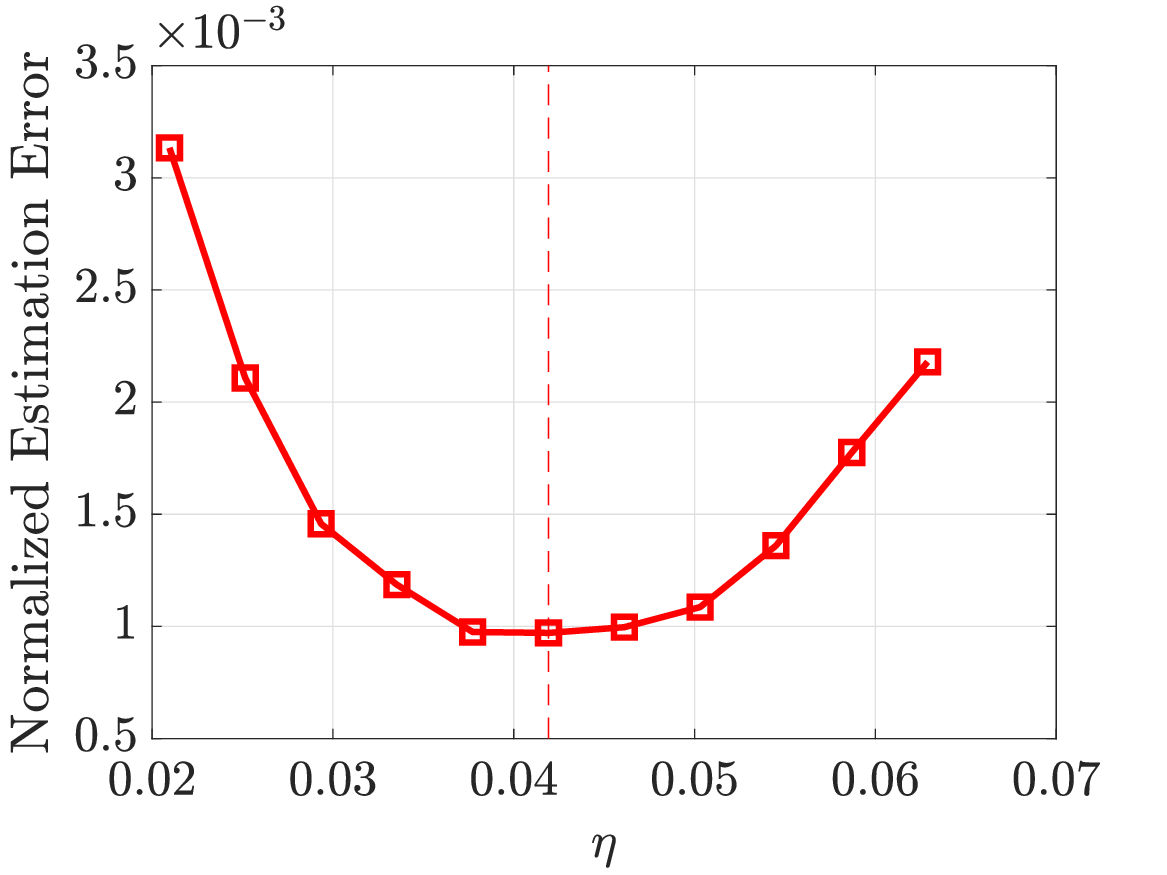}
    \end{minipage}
    \hfill 
    \begin{minipage}[b]{0.4\textwidth}
        \includegraphics[width=\textwidth]{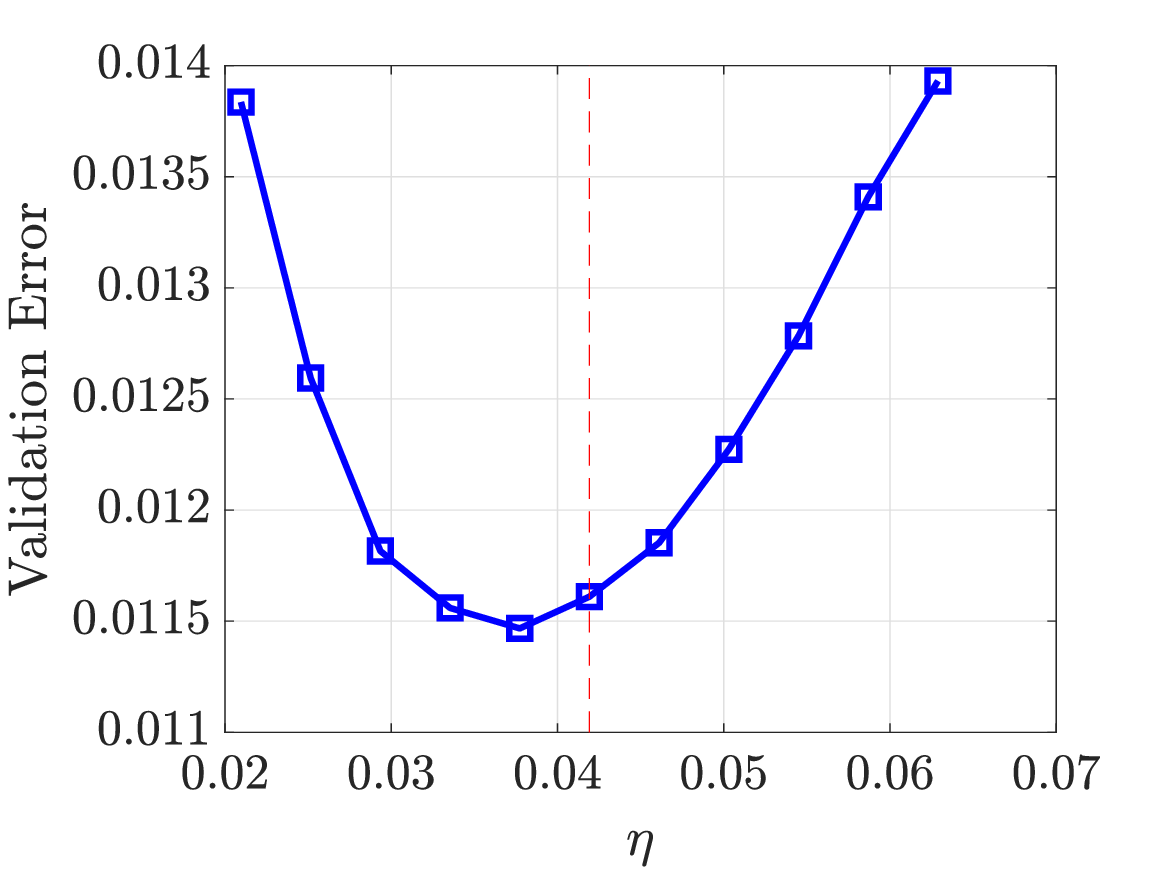}
    \end{minipage}
    
    \caption{Estimation error and validation error via cross-validation by AR for varying $\eta$ ($k=3, p=30,$ and $n=1,500$): The dotted vertical line indicates the location of $\eta_\star$ that achieves the equality in \eqref{eq:cond_eta}.}
    \label{fig:cross_val}
\end{figure}}

First, we consider the i.i.d. Gaussian noise model, i.e. $y_i=f_i(\mb \beta_{\star})+\epsilon_i$, where $\{\epsilon_i\}_{i=1}^n$ are i.i.d  following $\mathrm{Normal}(0,\sigma^2)$. 
To track the change of the estimation performance as a function of the noise strength, the dimension parameters are fixed as $p = 30$ and $k=6$. 
AM has shown to be consistent, with an error rate that vanishes as $n$ grows{\cite{ghosh2021max}}. 
Its empirical estimation error decays similarly in the experiment. 
However, we observe that AR has a larger estimation error compared to AM, which remains nontrivial even for large $n$. 
We conjecture that this bias term is due to the regularizer with an imperfect anchor vector. 
In fact, as the anchor vector is obtained from a more accurate initial estimate, the result estimation error decays accordingly. 
Motivated by this observation, we consider a modification of AR with further iterative refinements, which we call the \textit{iterative anchored regression} (IAR). 
The first iteration of IAR is equivalent to 
AR, but in the subsequent iterations, the anchor vector is refined by using the estimate from the previous iteration. The entire IAR algorithm is summarized in \Cref{ItrAR}. The number of iterations in IAR is set to $I_{\mathrm{IAR}}=40$. 
Figure~\ref{fig:gaus_noise} shows that with more iterations the performance of iterative AR becomes as good as that of AM. Moreover, for small $n$ (e.g. $n \leq 1,000$), IAR provides a smaller estimation error than AM.
{Moreover, we also study the sensitivity to the choice of the parameter $\eta$ in \eqref{eq:estimator}. 
The need to tune this parameter can be a weakness of AR since AM does not involve any such parameter. 
As shown in \Cref{fig:cross_val}, the estimation error by AR does not critically depend on $\eta$. In this experiment, we vary $\eta$ around $\eta_\star$ that achieves the equality in \eqref{eq:cond_eta} with $\pm 50\%$ margin. Within this range, the estimation error remains small. 
Also, note that the minimum estimation error is achieved when $\eta$ is slightly smaller than $\eta_\star$. 
It still remains to set the value of $\eta$ within this range. 
Since the observations are corrupted with i.i.d. noise in this experiment, we applied a 5-fold cross-validation to estimate the validation error. 
\Cref{fig:cross_val} suggests that choosing an $\eta$ value that yields the smallest prediction error will likely result in the smallest estimation error.}

 \begin{figure*}
\center
\hfill
\subfigure[$\varphi=0.1$]{
\includegraphics[scale=0.28]{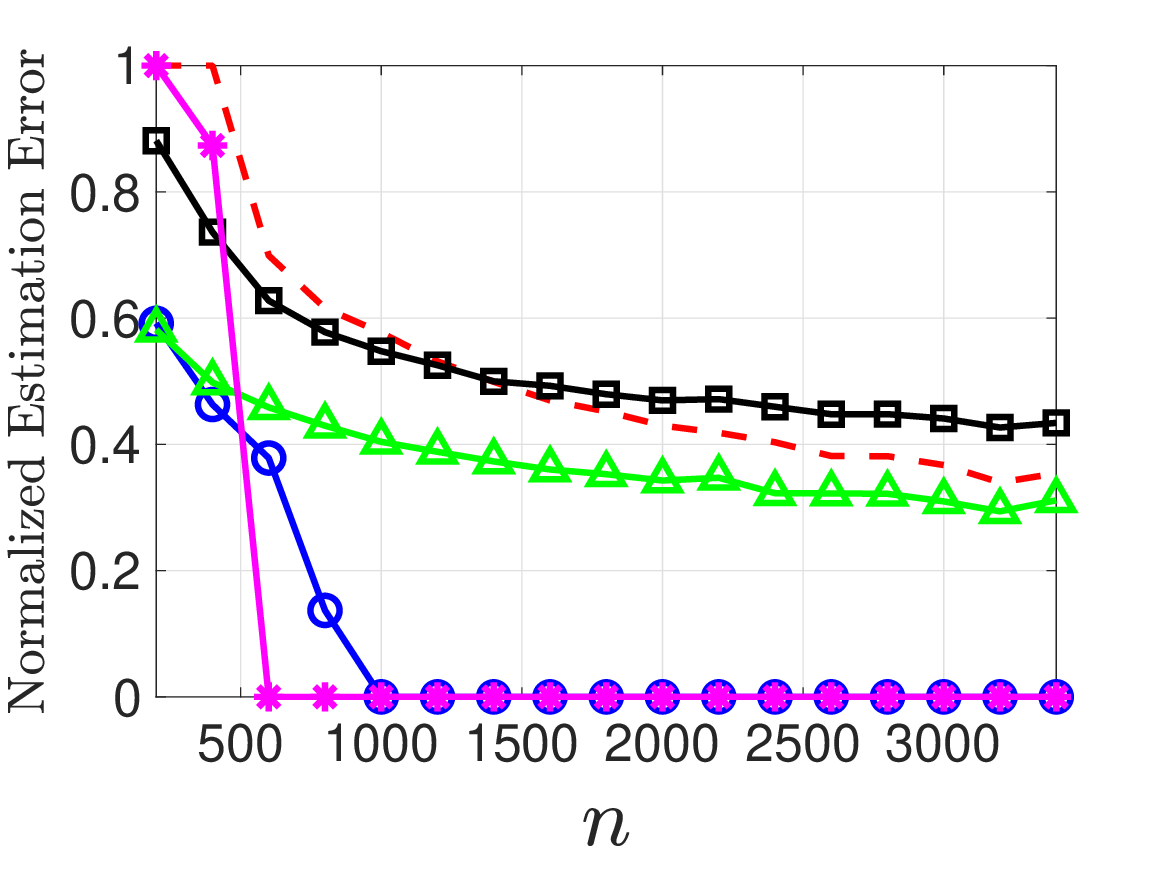}}
\hfill
\subfigure[$\varphi=0.2$]{
\includegraphics[scale=0.28]{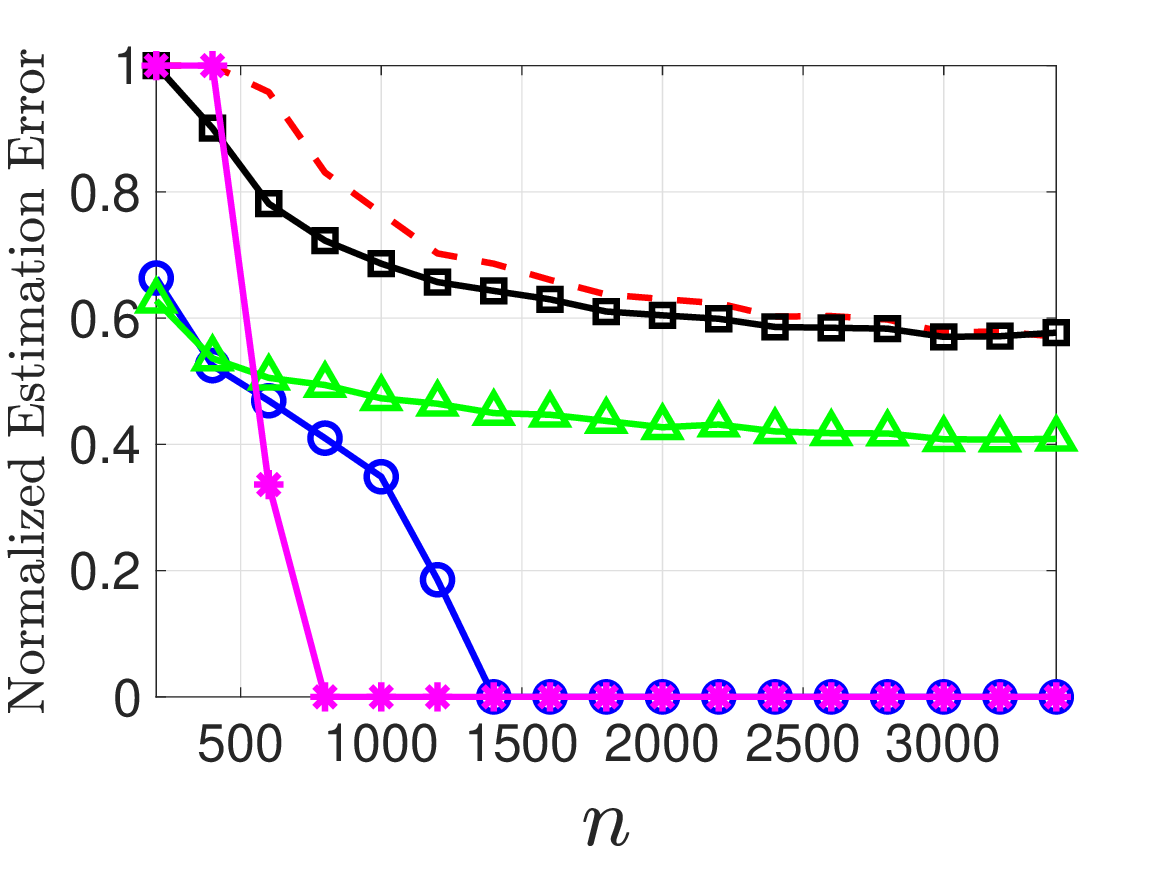}}
\hfill
\subfigure[$\varphi=0.3$]{
\includegraphics[scale=0.28]{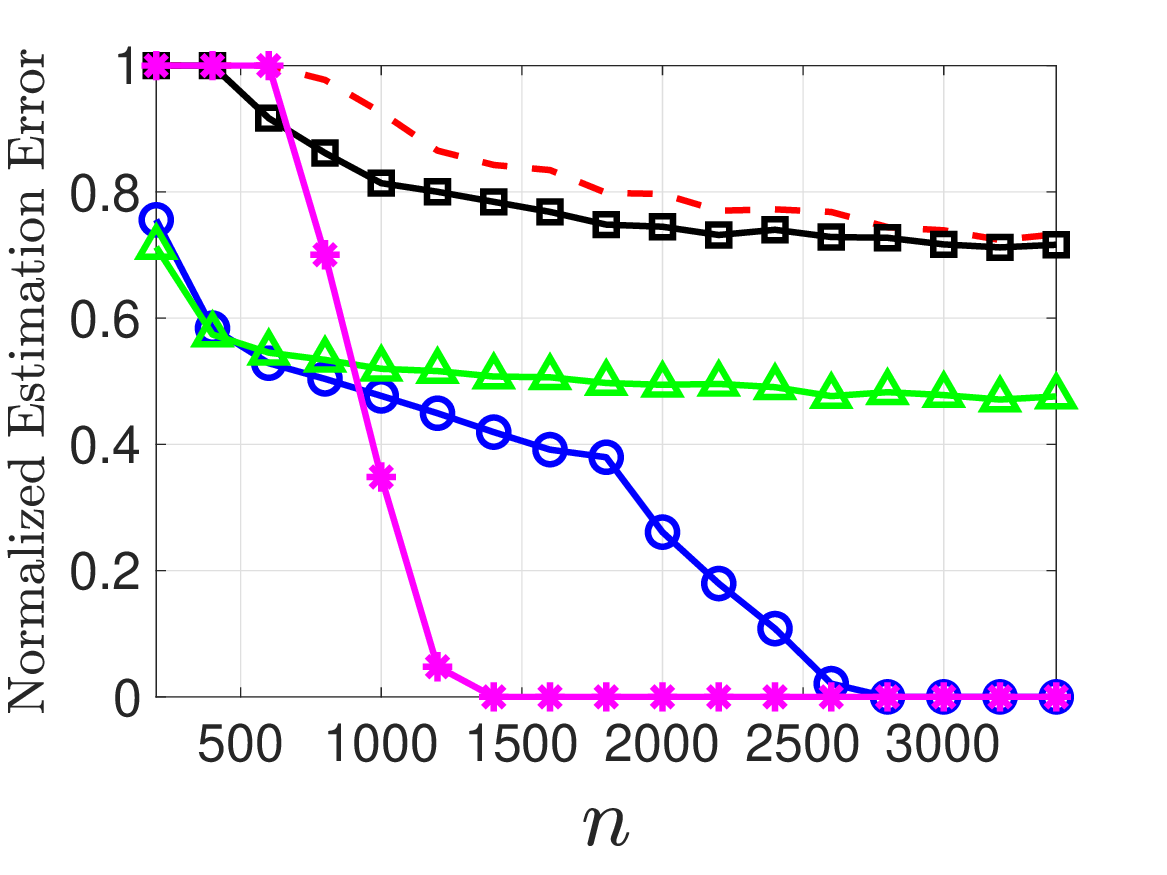}}
\hfill
\caption{Estimation error versus the number of observations $n$ under multiplicative Bernoulli noise model with probability $\varphi$ ($k=6$ and $p=30$): repeated random initialization (black line with square markers), AR (green line with triangle markers), IAR (blue line with circle markers), AM (red dashed line), {and AM-LAD (magenta line with asterisk markers)}. All methods start from repeated random initialization. 
}
\label{fig:mass_noise}
\end{figure*}


{
Next, we study the empirical performance of the estimators under a gross error model. 
In Section~\ref{sec:mainresult}, we have shown that the theoretical analysis of AR combined with the initialization by Ghosh et al. \cite{ghosh2021max} applies to this model.}
Specifically, each observation is corrupted by a sparse noise according to the multiplicative Bernoulli model with probability $\varphi$, that is, $\mathbb{P}\{y_i = -f_i(\mb \beta_\star)\} = \varphi$ and $\mathbb{P}\{y_i = f_i(\mb \beta_\star)\} = 1-\varphi$ for $i \in [n]$. The multiplicative Bernoulli noise model has a similarity with the Massart noise \cite[Definition 1.1]{diakonikolas2021relu}. 
{ 
Similar to the previous experiment, we compare AR to IAR and AM. 
Furthermore, we also study the performance of a variation of AM in which the least squares update is substituted by LAD. It will be denoted by AM-LAD.
}
\Cref{fig:mass_noise} illustrates the estimation error in this setting where $p=30$ and $k=6$. 
Unlike the case of Gaussian noise, AR outperforms AM in the presence of multiplicative Bernoulli noise. { 
Furthermore, IAR and AM-LAD achieve exact recovery over the range of $\varphi$ in this experiment. 
}

\section{Proof of Theorem~\ref{thm:main}}
We prove Theorem~\ref{thm:main} in two steps. First, in the following proposition, we present a sufficient condition for stable estimation by convex program in \eqref{eq:estimator}. 
Then we derive an upper bound on $\varrho$ in the proposition, which provides the sample complexity condition along with the corresponding error bound in Theorem~\ref{thm:main}.

\begin{prop}
\label{prop:main}
Under the hypothesis of Theorem~\ref{thm:main}, suppose that $\tilde{\mb \beta}$ satisfies
\begin{align}
\varrho & := \inf_{\begin{subarray}{l} j \in [k] \\ \mb w \in \mbb S^{p-1} \end{subarray}} \E \, \bbone_{\mc C_j}(\mb g) \left| \left\langle \mb g, \mb w \right\rangle \right| 
- \sup_{\begin{subarray}{l} j \in [k] \\ \mb w \in \mbb S^{p-1} \end{subarray}} \E \bbone_{\tilde{\mc C}_{j}\setminus \mc C_j}(\mb g) \langle \mb g, \mb w \rangle_+\nonumber\\
&\,\,\,\,\,\,- \sup_{\begin{subarray}{l} j \in [k] \\ \mb w \in \mbb S^{p-1} \end{subarray}} \E \bbone_{\mc C_j \setminus \tilde{\mc C}_{j}}(\mb g) \langle \mb g, \mb w \rangle_+ > 0\,.
\label{eq:def_varrho}
\end{align}
Then there exists an absolute constant $c>0$ such that if
\begin{equation}
\label{prop:samplecom_noise}
n \geq c \varrho^{-2} \left( 4p \log^3p \log^5k + 4\log(\delta^{-1}) \log k \right)\,,
\end{equation}
then the solution $\hat{\mb \beta}$ to the optimization problem in \eqref{eq:estimator} obeys
\begin{equation}
\label{prop:err_bnd}
\sum_{j=1}^{k}\|\mb \beta_{\star,j}-\hat{\mb \beta}_j\|_2 \leq \frac{2}{\varrho n}\sum_{i=1}^{n}|w_i|
\end{equation}
with probability $1 - \delta$.
\end{prop}

\begin{IEEEproof}
{
We first show that there exists a constant $c > 0$ such that the condition in \eqref{eq:initial_condition} implies $\zeta>0$. Hence, we consider 
\begin{equation}
\label{eq:alt_initial_condition}
 \underbrace{\min_{j\in[k]}
\sqrt{\frac{\pi}{32}} \, \P^2\{\mb g\in {\mc C_j}\}}_{\mathrm{(i)}}-\underbrace{2\max_{j\in[k]}  \sqrt{\P\{\mb g\in\tilde{\mc C}_j\triangle{\mc C_j}\}}}_{\mathrm{(ii)}} > 0.
\end{equation}
It follows from the definition of $\pi_{\min}$ that (i) in \eqref{eq:alt_initial_condition} is bounded from below as
\begin{equation}
\label{eq:lwb_a}
\mathrm{(i)} \geq \sqrt{\frac{\pi}{32}}\pi_{\min}^2.
\end{equation}
It only remains to find an appropriate upper bound on (ii). 
Since $\{\mc{C}_j\}_{j=1}^k$ consists of disjoint sets (except their boundaries corresponding to sets of measure zero), for a fixed $j \in [k]$, the symmetric difference between $\tilde{\mc{C}}_j$ and $\mc{C}_j$ is written as
\[
\tilde{\mc{C}}_j\triangle{\mc C_j}=\left(\cup_{j'\neq j}\tilde{\mc{C}}_j\cap \mc{C}_{j'}\right)\cup\left(\cup_{j'\neq j}\mc{C}_j\cap\tilde{\mc C}_{j'}\right).
\]
Therefore, we obtain 
\begin{equation}
\label{eq:ub_b1}
\mathrm{(ii)} \leq 2 \sqrt{2k} \, \max_{j \in [k]} \max_{j' \in [k] \setminus \{j\}} \sqrt{\P\left(\mb g\in\tilde{\mc{C}}_j\cap{\mc{C}}_{j'}\right)}.
\end{equation}
Moreover, since
\begin{equation}
\begin{aligned}
\mb g\in \tilde{\mc{C}}_j\cap \mc{C}_{j'}
&\implies \mb g^\T \tilde{\mb \beta}_j \geq \mb x_i^\T \tilde{\mb \beta}_{j'}, ~ \mb x_i^\T \mb \beta_{\star,j'} \geq \mb x_i^\T \mb \beta_{\star,j} \\
&\implies \mb g^\T  (\tilde{\mb \beta}_j-\tilde{\mb \beta}_{j'} )\geq 0, ~ \mb g^\T(\mb \beta_{\star,j}-\mb \beta_{\star,j'})\leq 0 \\
&\implies \mb g^\T  (\tilde{\mb \beta}_j-\tilde{\mb \beta}_{j'} )\cdot \mb g^\T(\mb \beta_{\star,j}-\mb \beta_{\star,j'})\leq 0,
\end{aligned}
\label{eq:incursion}
\end{equation}
with \cite[Lemma~9]{ghosh2021max}, $\mathrm{(ii)}$ in \eqref{eq:ub_b1} is further upper-bounded by
\begin{equation}
\label{eq:b_bnd1}
\begin{aligned}
&\mathrm{(ii)}\\
&\leq 2\sqrt{2 k}\cdot\\
&\quad \max_{j \in [k]}\max_{j' \in [k] \setminus \{j\}} \sqrt{\P\left(\mb g^\T  (\tilde{\mb \beta}_j-\tilde{\mb \beta}_{j'}) \cdot \mb g^\T(\mb \beta_{\star,j}-\mb \beta_{\star,j'}) \leq 0\right)} \\
&\leq C \sqrt{k}\cdot\max_{j \in [k]} \max_{j' \in [k] \setminus \{j\}}\Bigg(\sqrt{\frac{\|(\tilde{\mb \beta}_j-\tilde{\mb \beta}_{j'})-({\mb \beta}_{\star,j}-{\mb \beta}_{\star,j'})\|_2}{\|{\mb \beta}_{\star,j}-{\mb \beta}_{\star,j'}\|_2}}\\
&\quad\cdot{\log^{1/4}\left(\frac{2\|{\mb \beta}_{\star,j}-{\mb \beta}_{\star,j'}\|_2}{\|(\tilde{\mb \beta}_j-\tilde{\mb \beta}_{j'})-({\mb \beta}_{\star,j}-{\mb \beta}_{\star,j'})\|_2}\right)}\Bigg),
\end{aligned}
\end{equation}
for an absolute constant $C > 0$. 
Then, by plugging in \eqref{eq:lwb_a} and \eqref{eq:b_bnd1} to \eqref{eq:alt_initial_condition}, we obtain a sufficient condition for \eqref{eq:alt_initial_condition} as
\begin{equation}
\label{eq:alt_initial_condition2}
\begin{aligned}
&C \sqrt{k} \max_{j \in [k]} \max_{j' \in [k] \setminus \{j\}} \sqrt{\frac{\|(\tilde{\mb \beta}_j-\tilde{\mb \beta}_{j'})-({\mb \beta}_{\star,j}-{\mb \beta}_{\star,j'})\|_2}{\|{\mb \beta}_{\star,j}-{\mb \beta}_{\star,j'}\|_2}}\cdot\\
&{\log^{1/4}\left(\frac{2\|{\mb \beta}_{\star,j}-{\mb \beta}_{\star,j'}\|_2}{\|(\tilde{\mb \beta}_j-\tilde{\mb \beta}_{j'})-({\mb \beta}_{\star,j}-{\mb \beta}_{\star,j'})\|_2}\right)} < \sqrt{\frac{\pi}{32}}\pi_{\min}^2.
\end{aligned}
\end{equation}
For a fixed $j' \in [k] \setminus \{j\}$, let 
\[
a = \frac{\|(\tilde{\mb \beta}_j-\tilde{\mb \beta}_{j'})-({\mb \beta}_{\star,j}-{\mb \beta}_{\star,j'})\|_2}{\|{\mb \beta}_{\star,j}-{\mb \beta}_{\star,j'}\|_2} 
\quad \text{and} \quad 
b = \frac{\pi_{\min}^4}{k}.
\]
Since $a,b\in(0,0.1]$ and $a\leq\frac{b}{2}\log^{-1/2}(1/b)$ imply $a\log^{1/2}(2/a)\leq b$, if one chooses $c$ in \eqref{eq:initial_condition} so that $c < \frac{\pi}{32 C^2}$, then \eqref{eq:initial_condition} implies \eqref{eq:alt_initial_condition2} for all distinct $j,j' \in [k]$. In the remainder of the proof, we will assume that \eqref{eq:alt_initial_condition} holds.}

We show that, for a sufficiently large $\rho > 0$, the following three conditions cannot hold simultaneously: 
\begin{align}
    & \frac{1}{n} \sum_{i=1}^n \left(f_i(\mb \beta_{\star}+\mb z) - y_i\right)_+ \leq \eta\,, \label{eq:proof:prop1:cond1} \\
    &\norm{\mb z}_{1,2} > \rho\,, \label{eq:proof:prop1:cond2} \\
    & \langle \mb \theta, \mb z \rangle \geq 0\,. \label{eq:proof:prop1:cond3}
\end{align}
Therefore, assuming \eqref{eq:proof:prop1:cond2} and \eqref{eq:proof:prop1:cond3} hold, it suffices to show 

\begin{equation}
\mc L(\mb z) := \frac{1}{n} \sum_{i=1}^n \left(f_i(\mb \beta_{\star}+\mb z) - y_i\right)_+ > \eta\,. \label{eq:proof:prop1:cond1:neg}
\end{equation}
To this end, we derive a lower bound on $\mc L(\mb z)$ as follows:
\begin{align}
\mc L(\mb z) 
&\geq \frac{1}{n}\sum_{i=1}^{n} \left(f_i(\mb \beta_{\star}+\mb z) - f_i(\mb \beta_\star)\right)_+ - \frac{1}{n} \sum_{i=1}^n (w_i)_+ \nonumber \\
&\overset{\mathrm{(a)}}{\geq} \frac{1}{n}\sum_{i=1}^{n} \left( \left\langle \nabla f_i(\mb \beta_\star), \mb z \right\rangle \right)_+ - \frac{1}{n} \sum_{i=1}^n (w_i)_+ \nonumber \\
&= \frac{1}{n}\sum_{i=1}^{n} \frac{\left|\left\langle \nabla f_i(\mb \beta_\star), \mb z \right\rangle \right|}{2} + \frac{1}{n}\sum_{i=1}^{n} \frac{\left\langle \nabla f_i(\mb \beta_\star), \mb z \right\rangle}{2}\nonumber \\
&\quad- \frac{1}{n} \sum_{i=1}^n (w_i)_+ \nonumber \\
&= \frac{1}{n}\sum_{i=1}^{n} \frac{\left|\left\langle \nabla f_i(\mb \beta_\star), \mb z \right\rangle \right|}{2} + \frac{1}{n}\sum_{i=1}^{n} \frac{\left\langle \nabla f_i(\mb \beta_\star), \mb z \right\rangle}{2} \nonumber\\
&\quad- \langle \mb \theta, \mb z \rangle+ \langle \mb \theta, \mb z \rangle - \frac{1}{n} \sum_{i=1}^n (w_i)_+ \nonumber \\
&\overset{\mathrm{(b)}}{=}
\langle \mb \theta, \mb z \rangle - \frac{1}{n} \sum_{i=1}^n (w_i)_+ + \frac{1}{n} \sum_{i=1}^{n} \frac{\left|\left\langle \nabla f_i(\mb \beta_\star), \mb z \right\rangle \right|}{2}\nonumber\\
&\quad+ \frac{1}{n}\sum_{i=1}^{n} \frac{\langle \nabla f_i(\mb \beta_\star) - \nabla f_i(\tilde{\mb \beta}), \mb z \rangle}{2} \nonumber \\
&\overset{\mathrm{(c)}}{=}
\langle \mb \theta, \mb z \rangle - \frac{1}{n} \sum_{i=1}^n (w_i)_+ + \frac{1}{n}\sum_{i=1}^{n}\sum_{j=1}^{k}\frac{\bbone_{\mc C_j}(\mb x_i)|\langle\mb x_i,\mb z_j\rangle|}{2}\nonumber\\
& \quad + \frac{1}{n} \sum_{i=1}^n \sum_{j=1}^k 
\frac{ \{
\bbone_{\mc C_j}(\mb x_i) - \bbone_{\tilde{\mc C}_j}(\mb x_i)
\} \langle \mb x_i, \mb z_j \rangle}{2} \label{eq:lowerbound_Lz}
\,,
\end{align}
where (a) holds by the convexity of $f_i$, which implies
\[
f_i(\mb \beta_{\star}+\mb z)\geq f_i(\mb \beta_{\star})+\langle\nabla{f_i(\mb \beta_{\star})},\mb z\rangle\,,
\]
(b) follows from \eqref{eq:anchorvector}, and (c) is obtained by calculating $\nabla f_i(\mb \beta)$ at $\mb \beta = \mb \beta_\star$ and $\mb \beta = \tilde{\mb \beta}$. 
We further proceed by obtaining lower bounds on the last two terms in \eqref{eq:lowerbound_Lz} by the following lemmas, which are proved in Appendices~\ref{sec:proof:lem:rhs} and \ref{sec:proof:lem:lhs}.

\begin{lem}
\label{lem:rhs}
Let $(V_{\mb z})_{\mb z \in \mbb R^{kp}}$ be a random process defined by 
\[
V_{\mb z} := \frac{1}{n} \sum_{i=1}^n \sum_{j=1}^k \bbone_{\mc C_j}(\mb x_i) \left| \langle \mb x_i, \mb z_j \rangle \right|\,,
\]
where $\mb x_1,\dots,\mb x_n$ are i.i.d. $\mathrm{Normal}(\mb 0,\mb I_p)$. Then, for $\mb g \sim \mathrm{Normal}(\mb 0,\mb I_p)$ and any $\delta\in(0,1)$, there exists an absolute constant $c_1 > 0$ such that
\begin{align*}
\underline{V} := \inf_{\norm{\mb z}_{1,2} = 1} V_{\mb z}
\geq &\min_{j \in [k], \mb w \in \mbb S^{p-1}} \E \, \bbone_{\mc C_j}(\mb g) \left| \left\langle \mb g, \mb w \right\rangle \right|\\
&- c_1 \left(\frac{p \log^3p \log^5k + \log(\delta^{-1}) \log k}{n}\right)^{1/2}\,    
\end{align*}
holds with probability at least $1-\delta/2$.
\end{lem}
\begin{IEEEproof}
See \Cref{sec:proof:lem:rhs}.
\end{IEEEproof}

\begin{lem}
\label{lem:lhs}
Let $(Q_{\mb z})_{\mb z \in B_{1,2}}$ be a random process defined by 
\[
Q_{\mb z} := \frac{1}{n} \sum_{i=1}^n \sum_{j=1}^k 
\left\{
\bbone_{\tilde{\mc C}_j}(\mb x_i) - \bbone_{\mc C_j}(\mb x_i)
\right\} \langle \mb x_i, \mb z_j \rangle\,,
\] 
where $\mb x_1,\dots,\mb x_n$ are i.i.d. $\mathrm{Normal}(\mb 0,\mb I_p)$. Then, for $\mb g \sim \mathrm{Normal}(\mb 0,\mb I_p)$ and any $\delta\in(0,1)$, there exists an absolute constant $c_2 > 0$ such that
\begin{align*}
\overline{Q} := \sup_{\norm{\mb z}_{1,2} = 1} Q_{\mb z} 
&\leq 
\max_{j \in [k], \mb w \in \mbb S^{p-1}} \E \bbone_{\tilde{\mc C}_{j}\setminus \mc C_j}(\mb g) \langle \mb g, \mb w \rangle_+\\ 
&+ \max_{j \in [k], \mb w \in \mbb S^{p-1}} \E \bbone_{\mc C_j \setminus \tilde{\mc C}_{j}}(\mb g) \langle \mb g, \mb w \rangle_+ \\
& + c_2\left(\frac{p \log^3p \log^5k + \log(\delta^{-1}) \log k}{n}\right)^{1/2}\,
\end{align*}
holds with probability at least $1-\delta/2$.
\end{lem}
\begin{IEEEproof}
See \Cref{sec:proof:lem:lhs}.
\end{IEEEproof}

Since $V_{\mb z}$ are $Q_{\mb z}$ are homogeneous in $\mb z$, we obtain that the third term in the right-hand side of \eqref{eq:lowerbound_Lz} is written as $V_{\mb z}$ and lower-bounded by
\begin{equation}
\label{eq:defVz}
\frac{V_{\mb z}}{2} \geq \frac{\underline{V} \norm{\mb z}_{1,2}}{2}\,.
\end{equation}
Similarly, the last term in the right-hand side of \eqref{eq:lowerbound_Lz} is written as $- Q_{\mb z}$ and lower-bounded by
\begin{equation}
\label{eq:defQz}
- \frac{Q_{\mb z}}{2} \geq - \frac{\overline{Q} \norm{\mb z}_{1,2}}{2}\,.
\end{equation}
Furthermore, by Lemmas~\ref{lem:rhs} and \ref{lem:lhs}, the condition in \eqref{prop:samplecom_noise} implies that
\begin{equation}
\label{eq:proof:prop1:varrho}
\underline{V} - \overline{Q} \geq {c_3}\varrho > 0
\end{equation}
holds with probability $1-\delta$ { for an absolute constant $c_3>0$}.  
Then we choose $\rho$ so that it satisfies
{
\[
\rho = \frac{2}{\underline{V} - \overline{Q}} \cdot \left(\eta+\frac{1}{n} \sum_{i=1}^n (w_i)_+\right)\,.
\]}
Next, by plugging in the above estimates to \eqref{eq:lowerbound_Lz}, we obtain that, under the event in \eqref{eq:proof:prop1:varrho}, the conditions in \eqref{eq:proof:prop1:cond2} and \eqref{eq:proof:prop1:cond3} imply 
\begin{align*}
\mc L(\mb z) 
&\geq \langle\mb \theta, \mb z \rangle - \frac{1}{n} \sum_{i=1}^n (w_i)_+ + \frac{(\underline{V} -\overline{Q})\norm{\mb z}_{1,2}}{2} \\
&> - \frac{1}{n} \sum_{i=1}^n (w_i)_+ + \frac{(\underline{V}-\overline{Q})\rho}{2}\\
&={ -\frac{1}{n} \sum_{i=1}^n (w_i)_+ + \frac{1}{n} \sum_{i=1}^n (w_i)_+ +\eta } \\
&={ \eta\,.}
\end{align*}
This lower bound implies \eqref{eq:proof:prop1:cond1:neg}. 
Therefore we have shown that the three conditions in \eqref{eq:proof:prop1:cond1}, \eqref{eq:proof:prop1:cond2}, and  \eqref{eq:proof:prop1:cond3} cannot hold simultaneously. 
It remains to apply the claim to a special case. 

Let $\hat{\mb z} = \hat{\mb \beta} - \mb \beta_\star$. 
Recall that both $\hat{\mb \beta}$ and $\mb \beta_\star$ are feasible for the optimization problem in \eqref{eq:estimator}. 
Moreover, since $\hat{\mb \beta}$ is the maximizer, it follows that $\langle \mb \theta, \hat{\mb \beta}\rangle \geq \langle \mb \theta, \mb \beta_\star \rangle$, which implies $\langle \mb \theta, \hat{\mb z}\rangle \geq 0$. 
Therefore the conditions in \eqref{eq:proof:prop1:cond1} and \eqref{eq:proof:prop1:cond3} are satisfied with $\mb z$ substituted by $\hat{\mb z}$. 
Since the three conditions cannot be satisfied simultaneously, the condition in \eqref{eq:proof:prop1:cond2} cannot hold, i.e. $\hat{\mb z}$ satisfies
{
\begin{equation}
\label{eq:final_estbnd}
\norm{\hat{\mb z}}_{1,2} 
\leq \rho \leq { \frac{2}{\varrho}\left(\eta+\frac{1}{n}\sum_{i=1}^n (w_i)_+\right)}\,.
\end{equation}
Since the noise vector $\mb w$ was arbitrary, \eqref{eq:final_estbnd} holds for any $\mb w$. 
Furthermore, since the random processes in \Cref{lem:rhs} and \Cref{lem:lhs} do not depend on the noise $\mb w$, the conclusion of the theorem applies to an adversarial noise without amplifying the error probability. 
}
\end{IEEEproof}

Next, we use the following lemma to obtain a lower bound on $\varrho$ in \eqref{eq:def_varrho}. 

\begin{lem}
\label{lem:lowbound_expect}
Let $\mc A \subset \mbb R^p$ be of finite Gaussian measure and $\mb g \sim \mathrm{Normal}(\mb 0, \mb I_p)$. 
Then we have
\[
\inf_{\begin{subarray}{l} \mb w \in \mbb S^{p-1} \end{subarray}} \E \, \bbone_{\mc A}(\mb g) \left| \left\langle \mb g, \mb w \right\rangle \right| 
\geq \sqrt{\frac{\pi}{32}} \, \P^2\{\mb g\in {\mc A}\}\,
\]
and
\[
\sup_{\begin{subarray}{l} \mb w \in \mbb S^{p-1} \end{subarray}} \E \bbone_{\mc A}(\mb g) \langle \mb g, \mb w \rangle_+\leq\sqrt{\P\{\mb g\in \mc A\}}\,.
\]
\end{lem}
\begin{IEEEproof}
For an arbitrarily fixed $\epsilon > 0$, let $S_\epsilon \subset \mbb R^p$ denote the set defined by
\[
S_\epsilon := \{\mb x\in\mbb R^p : |\langle \mb x, \mb w \rangle| < \epsilon\}\,.
\]
Then we have 
\begin{align}
     \E \bbone_\mc{C}(\mb g)|\inp{\mb g, \mb w}| & \ge \epsilon\,\E \bbone_\mc{C}(\mb g)\bbone_{\mc S_\epsilon^{\mr c}}(\mb g) \nonumber \\ 
     & =  \epsilon\,\E\left(\bbone_\mc{C}(\mb g)-\bbone_\mc{C}(\mb g)\bbone_{\mc S_\epsilon}(\mb g)\right) \nonumber \\
     & \ge  \epsilon\,\E\left(\bbone_\mc{C}(\mb g)-\bbone_{\mc S_\epsilon}(\mb g)\right) \nonumber \\
     & = \epsilon\,\left(\P\{\mb g \in \mc C\} - \P\{\mb g \in \mc S_\epsilon\}\right)\,. \label{eq:lb_ind_exp}
\end{align}
Moreover, since $\langle\mb g,\mb w\rangle\sim\mathrm{Normal}(0,1)$, $\P\{\mb g\in S_\epsilon\}$ is upper-bounded by
\begin{align}
\P\{\mb g\in S_\epsilon\}
&=\P\{|\langle\mb g,\mb w\rangle|<\epsilon\}
=\int_{-\epsilon}^{\epsilon} \frac{1}{\sqrt{2\pi}} e^{-u^2/2} du
\leq\epsilon\sqrt{\frac{2}{\pi}}\,.
\label{eq:ub_g_in_Sepsilon}
\end{align}
By plugging in \eqref{eq:ub_g_in_Sepsilon} to \eqref{eq:lb_ind_exp}, we obtain
\begin{equation}
\E \bbone_\mc{C}(\mb g)|\inp{\mb g, \mb w}| 
\ge \epsilon\left(\P\{\mb g\in \mc C\}-\epsilon\sqrt{\frac{2}{\pi}}\right)\,. \label{eq:lb_ind_exp2}
\end{equation}
Since the parameter $\epsilon>0$ was arbitrary, one can we maximize the right-hand side of \eqref{eq:lb_ind_exp2} with respect to $\epsilon$ to obtain the tightest lower bound. 
Note that the objective is a concave quadratic function and the maximum is attained at $\epsilon=\sqrt{{\pi}/{8}} \, \P\left\{\mb g\in \mc C\right\}$. 
This provides the lower bound in the first assertion. 
Next, by the Cauchy-Schwarz inequality, we obtain the upper bound in the second assertion as follows:
\begin{align*}
\E \bbone_{\mc A}(\mb g) \langle \mb g, \mb w \rangle_+
&\leq\sqrt{\E\left(\bbone_{\mc A}(\mb g)\right)^2}\sqrt{\E{\langle\mb g,\mb w\rangle_+^2}}\\
&=\sqrt{\E\bbone_{\mc A}(\mb g)}\sqrt{\frac{\E{\langle\mb g,\mb w\rangle^2}}{2}}\\
&=\sqrt{\frac{\P\{\mb g\in\mc A\}}{2}}\,.
\end{align*}
\end{IEEEproof}

Finally, by applying Lemma~\ref{lem:lowbound_expect} to each of the expectation terms in $\varrho$, we obtain a lower bound on $\varrho$ given by
\begin{align}
\varrho &\geq \min_{j\in[k]}\sqrt{\frac{\pi}{32}} \, \P^2\left\{\mb g\in \mc C_j \right\}-\max_{j\in[k]}\sqrt{\P\left\{\mb g\in {\mc C_j}\setminus\tilde{\mc C}_j\right\}}\nonumber\\
&\quad-\max_{j\in[k]}\sqrt{\P\{\mb g\in \tilde{\mc C}_j\setminus\mc C_j\}}\, \nonumber \\
&\geq\min_{j\in[k]}\sqrt{\frac{\pi}{32}} \, \P^2\left\{\mb g\in \mc C_j \right\}-2\max_{j\in[k]}\sqrt{\P\left\{\mb g\in {\mc C_j}\triangle\tilde{\mc C}_j\right\}}\,, \label{eq:proof_lowb1}
\end{align}
where the second inequality holds since $\tilde{\mc C}_j\triangle{\mc C_j} = (\tilde{\mc C}_j\setminus\mc C_j) \cup (\mc C_j\setminus\tilde{\mc C}_j)$ for all $j \in [k]$. This implies that \eqref{thm:samplecom_noise} is a sufficient condition for \eqref{prop:samplecom_noise}. Moreover, substituting $\varrho$ in \eqref{prop:err_bnd} by the lower bound in \eqref{eq:proof_lowb1} provides \eqref{thm:err_bnd}. This completes the proof of Theorem~\ref{thm:main}.

{
\subsection{Tightness of the lower bound on $\varrho$}
In \eqref{eq:proof_lowb1}, we obtain a lower bound on $\varrho$ by Lemma~\ref{lem:lowbound_expect}. We show through the following example that the lower bound is tight in terms of its dependence on $\P\left\{\mb g\in \mc C_j \right\}$ for $j \in [k]$.}

\begin{exmp}
\label{exmp:mainexample}
Let $p=2$. Then $\tilde{\mc C}_{j}\setminus \mc C_j$ and $\mc C_j \setminus \tilde{\mc C}_{j}$ are Lorentz cones.  
Let $\theta_{\mc C_j}$, $\theta_{\mc C_j\setminus\tilde{\mc C}_j}$ and ${\theta_{\tilde{\mc C}_j\setminus{{\mc C_j}}}}$ denote the angular width of $\mc C_j$, ${\mc C}_{j}\setminus \tilde{\mc C}_j$, and $\tilde{\mc C}_j \setminus {\mc C}_{j}$ respectively. 
Furthermore, we assume that 
\begin{equation}
\label{eq:minprobCgeqmaxprobdiff}
\min_{j\in[k]} \P\left\{\mb g\in \mc C_j \right\} \geq \max_{j\in[k]} \P\left\{\mb g\in {\mc C_j}\triangle\tilde{\mc C}_j\right\}\,.
\end{equation}
In this case, the parameter $\varrho$ in Proposition \ref{prop:main} is expressed as
\begin{align}
\varrho=\frac{\sqrt{2}\Gamma(3/2)}{\Gamma(1)}\Bigg[&\min_{j\in[k]}\frac{2}{\pi}\sin^2\left(\frac{\theta_{\mc C_j}}{4}\right)
-\max_{j\in[k]}\frac{1}{\pi}\sin\left({\frac{\theta_{\tilde{\mc C}_j\setminus{\mc C_j}}}{2}}\right)\nonumber\\
&-\max_{j\in[k]}\frac{1}{\pi}\sin\left({\frac{\theta_{\mc C_j \setminus \tilde{\mc C}_j}}{2}}\right)\Bigg]\,.
\label{eq:exp_varrho_2D}
\end{align}
When $\theta_{\mc C}$ is small enough, $\sin(\theta_{\mc C})\approx{\theta_{\mc C}}$ holds by the Taylor series approximation. Hence, there exists absolute constants $c_1>0$ and $c_2>0$ such that
\[
\varrho=c_1\min_{j\in[k]}{\P^2\{\mb g\in \mc C_j\}}-c_2\max_{j\in[k]}\P\{\mb g\in\tilde{\mc C}_j\triangle{\mc C_j}\}\,.
\] 
This example shows that $\zeta$ in Theorem~\ref{thm:main} is tight in the sense that the dominating term in both $\varrho$ and $\zeta$ is proportional to the squared probability measure of the smallest $\mc C_j$.
\end{exmp}

Let $\theta_{\mc C_j}$ denote the angular width of $\mc C_j$. 
Without loss of generality, we may assume that $\min_{j\in[k]} \theta_{\mc C_j} \leq \pi$. 
Furthermore, the assumption in \eqref{eq:minprobCgeqmaxprobdiff} implies that the angular width of $\mc C_j \triangle \tilde{\mc C}_j$ is at most $\pi$ for all $j \in [k]$. 
Therefore, the identity in \eqref{eq:exp_varrho_2D} is obtained by applying the following lemma, proved in Appendix~\ref{sec:proof:lem:Expect}, to the infimum/supremum of expectation terms in \eqref{eq:def_varrho}.

\begin{lem}
\label{lem:Expect}
Let $\mc C$ be a polyhedral cone in $\mbb R^2$ and $\mb g\sim \mathrm{Normal}(\mb 0, \mb I_2)$. Suppose that the angular width of $\mc C$, denoted by $\theta_{\mc C}$ satisfies $0\leq\theta_{\mc C}\leq\pi$.
Then we have
\[
\inf_{\begin{subarray}{l}\mb w \in \mbb S^{1} \end{subarray}} \E \, \bbone_{\mc C}(\mb g) \left| \left\langle \mb g, \mb w \right\rangle \right|=\frac{2\sqrt{2}\Gamma({3}/{2})}{\pi\Gamma(2)}\sin^2\left(\frac{\theta_{\mc C}}{4}\right)\,
\]
and
\[
\sup_{\begin{subarray}{l} \mb w \in \mbb S^{1} \end{subarray}} \E \bbone_{{\mc C}}(\mb g) \langle \mb g, \mb w \rangle_+=\frac{\sqrt{2}\Gamma({3}/{2})}{\pi\Gamma(2)}\sin\left({\frac{{\theta_{\mc C}}}{2}}\right)\,.
\]
\end{lem}
\begin{IEEEproof}
See \Cref{sec:proof:lem:Expect}.
\end{IEEEproof}

\section{Discussion}
{As discussed in Section~\ref{sec:experiment}, the proposed convex estimator provides a comparable error bound relative to an oracle estimator in the adversarial noise case. 
However, it does not provide a consistent estimator with random noise. 
{ This inconsistency arises due to the maximization of the correlation with the anchor vector $\mb \theta$. 
Since the direction of the anchor vector does not coincide with the ground truth, the convex estimator introduces a bias.} 
{ { As a way to mitigate } the bias in the convex estimator, we propose the iterative anchored regression that recursively refines the anchor vector to better align its direction with that of the ground truth. We have demonstrated that the iterative anchored regression empirically provides an exact recovery of the ground-truth parameters in the presence of outliers.} {Hence}, it would be fruitful to pursue the theoretical analysis of the iterative anchored regression, particularly in terms of its behavior in the presence of outliers and { random noise}. 
{ Each iteration solves a linear program, which costs $\tilde{O}\left(((n+p)k)^c\right)$ with $c\approx2.38$ { as discussed in \Cref{subsec:comparison}}}.
Therefore, the per-iteration cost of the iterative anchored regression { might be} higher than that of the alternating minimization, which is  ${O}(nkp^2)$. 
To further alleviate the computational cost of the iterative version, one might consider warm-start strategies in interior-point methods for linear programming (e.g.  \cite{john2008implementation}).



\appendix

{
\subsection{Proof of \Cref{lem:oracle_bound}}
\label{sec:proof_oracle}
For brevity, we introduce the shorthand notations
\[
\mb A_j=\sum_{i=1}^n \bbone_{\mc C_j}(\mb x_i) \mb x_i\mb x_i^\T, \quad \text{and} \quad \mb b_j = \sum_{i=1}^n \bbone_{\mc C_j}(\mb x_i) y_i\mb x_i.
\] 
Then, since each $\mc C_j$ is given by the intersection of $(k-1)$ half-planes in $\mathbb{R}^p$, by \cite[Theorem~2]{vapnik2015uniform}, it holds with probability at least $1-\delta/3$ that
\begin{equation}
\label{eq:lb_sel_Cj}
\begin{aligned}
&\sup_{j \in [k]} \left| \frac{1}{n} \sum_{i=1}^n \bbone_{\mc C_j}(\mb x_i) - \P(\mb g \in \mc C_j) \right|
\leq \\
&\qquad\qquad C_1 \sqrt{\frac{\log(3/\delta)+kp\log(n/p)}{n}},
\end{aligned}
\end{equation}
which implies
\begin{equation}
\label{eq:lb_sel_Cj2}
c_2 n \pi_{\min} 
\leq 
\sum_{i=1}^n \bbone_{\mc C_j}(\mb x_i) 
\leq 
C_3 n \pi_{\max}, \quad \forall j \in [k]. 
\end{equation} 
Moreover, by \cite[Theorem~5.7]{tan2019phase}, with probability at least $1-\delta/3$, we have
\[
\sup_{\mathcal{I}:|\mathcal{I}|\leq \alpha n}\lambda_{\max}\left(\sum_{i\in\mathcal{I}}\mb x_{i}\mb x_{i}^\T\right)
\leq C_4 \sqrt{\alpha} n
\] 
provided 
\begin{equation}
\label{eq:samp_tanver}
n\geq \max\left(p,\frac{\log(3/\delta)}{\alpha}\right).
\end{equation}
We also use the following claim: If 
\begin{equation}
\label{eq:lwb_trunc_sample}
n \geq C_5 \beta^{-2}\max(p\log(n/p),\log(3/\delta)),
\end{equation}
then it holds with probability $1-\delta/3$ that
\begin{equation}
\label{eq:lb_lambda_A_j}
\inf_{\mc I \subset [n] : |\mc I|\geq\beta n} \lambda_{\min} \left( \sum_{i\in \mc I} \mb x_i \mb x_i^\top \right) \geq c_6 n \beta^3.
\end{equation}

\begin{IEEEproof}[Proof of Claim] 
For an arbitrarily fixed $T > 0$, we have 
\begin{equation}
\label{eq:mainarg2}
\frac{1}{n} \sum_{i\in \mc I} \langle\mb \xi_i,\mb v\rangle^2
\geq \frac{\beta T}{2}, \quad \forall \mc I \subset [n] : |\mc I| \geq \beta n
\end{equation}
provided 
\begin{equation}
\label{eq:condition1}
N(\mb v) :=\sum_{i=1}^{n} \bbone_{\{\mb x: \langle \mb x, \mb  v\rangle^2>T\}}(\mb x_i) > n-\frac{\beta n}{2}.
\end{equation}
Since $\{\mb x: \langle \mb x, \mb  v\rangle^2>T\}$ is consists of two half-spaces in $\mbb{R}^p$, by \cite[Theorem~2]{vapnik2015uniform}, there exists an absolute constant $C_7 > 0$, for which it holds with probability at least $1-\delta/3$ that 
\begin{equation}
\label{eq:VCarg1}
\frac{1}{n}N(\mb v)
\geq \frac{1}{n}\E N(\mb v) - C_7 \sqrt{\frac{p\log(n/p)+\log(3/\delta)}{n}}, \,\, \forall \mb v \in \mathbb{S}^{p-1}.
\end{equation}
Moreover, due to \cite[Lemma~15]{ghosh2019max}, we have 
\begin{equation}
\label{eq:smallball1}
\frac{1}{n} \E N(\mb v) = \P\left(|\langle\mb x,\mb v\rangle|^2>T\right)\geq 1 - \sqrt{eT}.
\end{equation} 
Plugging in \eqref{eq:smallball1} into \eqref{eq:VCarg1} yields 
\[
\frac{1}{n} N(\mb v)\geq 1 - \sqrt{eT} - C_7 \sqrt{\frac{p\log(n/p)+\log(3/\delta)}{n}}, \,\, \forall \mb v \in \mathbb{S}^{p-1}.
\]
Then \eqref{eq:condition1} is satisfied for all $\mb v \in \mathbb{S}^{p}$ by $T= \frac{\beta^2}{16e}$ and $C_5 = (4C_7)^2$. 
\end{IEEEproof}
Since \eqref{eq:sample_comp_oracle} implies \eqref{eq:samp_tanver} and \eqref{eq:lwb_trunc_sample}, combining the above results provides that
\[
c_8 n \pi_{\min}^3
\leq 
\lambda_{\min}(\mb A_j)
\leq 
\lambda_{\max}(\mb A_j) 
\leq C_9 n \sqrt{\pi_{\max}}, \quad \forall j \in [k],
\]
holds with probability $1-\delta$. 
Then the least squares solution in \eqref{eq:decoupled_leastsquared} is written as $\hat{\mb \beta}_j=\mb A_j^{-1} \mb b_j$ and satisfies
\begin{equation}
\label{eq:bstar-b}
\begin{aligned}
&\|\mb \beta_{\star,j}-\hat{\mb \beta}_j\|_2
\geq \frac{\lambda_{\min}^{1/2}\left(\mb A_j\right)}{\lambda_{\mathrm{max}}(\mb A_j)}\left\|\left(w_i\right)_{i:\mb x_i\in\mathcal{C}_j}\right\|_2 \\
&\qquad\geq\frac{c_{10} \pi_{\min}^{3/2}}{\sqrt{n \pi_{\max}}} \left\|\left(w_i\right)_{i:\mb x_i\in\mathcal{C}_j}\right\|_2\geq 
{\frac{c_{11} \pi_{\min}^{3/2}}{\pi_{\max}} \left\|\left(w_i\right)_{i:\mb x_i\in\mathcal{C}_j}\right\|_\infty}.
\end{aligned}
\end{equation}
Then taking a sum over $j
\in [k]$ and maximizing over $\mb w$ satisfying { $\|\mb w\|_\infty \leq \eta'$}, we obtain 
\[
{\sup_{\|\mb w\|_\infty \leq \eta'}} \sum_{j=1}^k\|\mb \beta_{\star,j}-\hat{\mb \beta}_j\|_2 \geq \frac{c_{12} \pi_{\min}^{3/2} \eta'}{\pi_{\max}}.
\]
This completes the proof. 

\subsection{Proof of Lemma~\ref{lem:Expect}}
\label{sec:proof:lem:Expect}
We first prove the first assertion. 
Since $\mc C$ is a cone, it follows that $\mb g \in \mc C$ if and only $\mb g/\norm{\mb g}_2 \in \mc C$.
Moreover, Bayes' rule implies
\[
\E \, \bbone_{\mc C}(\mb g) \left| \left\langle \mb g, \mb w \right\rangle \right|
= 
\P\left\{\mb g\in {\mc C}\right\} \E \left[ \left. \left| \left\langle \mb g, \mb w \right\rangle \right| \,\right|\, \mb g \in \mc C \right]\,.
\]
Therefore we have
\begin{align}
&\inf_{\mb w \in \mbb S^{1} } \E \, \bbone_{\mc C}(\mb g) \left| \left\langle \mb g, \mb w \right\rangle \right|\nonumber\\
&=\inf_{\mb w \in \mbb S^{1} } \, \P\left\{\mb g\in {\mc C}\right\} \E \left[\|\mb g\|_2 \left| \langle \frac{\mb g}{\|\mb g\|_2}, \mb w \rangle\right| \,\left|\, \frac{\mb g}{\|\mb g\|_2}\in {\mc C}\right.\right] \nonumber \\
&\overset{\mathrm{(a)}}{=}\inf_{\begin{subarray}{l}  \mb w \in \mbb S^{1} \end{subarray}} \, \P\left\{\mb g\in {\mc C}\right\} \E \left[{\|\mb g\|_2}\right]\E \left[ \, \left| \left\langle \frac{\mb g}{\|\mb g\|_2}, \mb w \right\rangle\right| \,\left|\,\frac{\mb g}{\|\mb g\|_2}\in {\mc C}\right.\right] \nonumber \\
&\overset{\mathrm{(b)}}{=}{\frac{\sqrt{2}\Gamma(3/2)}{\Gamma(2)}} \inf_{\mb w \in \mbb S^{1} } \, {\frac{\theta_{\mc C}}{2\pi}} \, \E \left[ \, \left| \left\langle \frac{\mb g}{\|\mb g\|_2}, \mb w \right\rangle\right| \,\left|\,\frac{\mb g}{\|\mb g\|_2}\in {\mc C}\right.\right]\,, \label{eq1:proof:lem:Expect}
\end{align}
where (a) holds since $\|\mb g\|_2$ and $\mb g/\|\mb g\|_2$ are independent and (b) follows from $\E \norm{\mb g}_2 = \sqrt{2}\Gamma(3/2)/\Gamma(2)$ and \[
\P\{\mb g\in\mc C\}=\P\left\{\frac{\mb g}{\|\mb g\|_2}\in\mc C\right\}=\frac{\theta_{\mc C}}{2\pi}\,.
\]
Then it remains to compute the expectation in \eqref{eq1:proof:lem:Expect}. 
Below we show that 
\begin{equation}
\label{eq2:proof:lem:Expect}
\inf_{\mb w \in \mbb S^2}
\E \left[ \left| \left\langle \frac{\mb g}{\|\mb g\|_2}, \mb w \right\rangle\right| \,\left|\, \frac{\mb g}{\|\mb g\|_2}\in {\mc C} \right. \right] 
= \frac{4}{\theta_{\mc C}}\sin^2\left(\frac{\theta_{\mc C}}{2}\right)\,
\end{equation}
and
\begin{equation}
\label{eq2a:proof:lem:Expect}
\sup_{\mb w \in \mbb S^2}
\E \left[ \left| \left\langle \frac{\mb g}{\|\mb g\|_2}, \mb w \right\rangle\right| \,\left|\, \frac{\mb g}{\|\mb g\|_2}\in {\mc C} \right. \right] 
= \frac{2}{\theta_{\mc C}}\sin\left(\frac{\theta_{\mc C}}{2}\right)\,.
\end{equation}
Let $\mc T= \{\mb a, \mb b \} \subset \mbb S^{1}$ satisfy that $\mc C$ is the conic hull of $\mc T$. 
Then let $\mb z$ be the unit vector obtained by normalizing $(\mb a+\mb b)/2$. 
Then we have $\angle({\mb a,\mb z})=\theta_{\mc C}/2$ and $\angle(\mb b, \mb z)=\theta_{\mc C}/{2}$. Let $\phi: \mbb S^1 \to \mbb R$ be defined by $\phi(\mb w):=\angle(\mb z,\mb w)$. 
Since the conditional expectation applies to $|\langle \mb g / \norm{\mb g}_2, \mb w\rangle|$, which is invariant under the global sign change in $\mb w$, it suffices to consider $\mb w$ that satisfies $0\leq\phi(\mb w)\leq\pi$. 
Since $\mb g/\|\mb g\|_2$ is uniformly distributed on the unit sphere, the expectation term in \eqref{eq2:proof:lem:Expect} is written as
\begin{equation}
\label{eq:ojbective_cos}
\E \left[ \left| \left\langle \frac{\mb g}{\|\mb g\|_2}, \mb w \right\rangle\right| \,\left|\, \frac{\mb g}{\|\mb g\|_2}\in {\mc C} \right. \right]
= \frac{1}{\theta_{\mc C}}\int_{\phi(\mb w)-\theta_{\mc C}/2}^{\phi(\mb w)+\theta_{\mc C}/2} |\cos\theta| d\theta\,.
\end{equation} 
It follows from the assumption on the range of $\theta_{\mc C}$ and $\phi(\mb w)$ that $-{\pi}/{2}\leq\phi(\mb w)-{\theta_{\mc C}}/{2}\leq\pi$ and $0\leq\phi(\mb w)+{\theta_{\mc C}}/{2}\leq{3\pi}/{2}$. 
We proceed by separately considering the complementary cases for $(\theta_{\mc C}, \phi(\mb w))$ given below.\\

\noindent\textbf{Case 1:} Suppose that 
\begin{equation}
\label{eq:case1}
-\frac{\pi}{2}\leq\phi(\mb w)-\frac{\theta_{\mc C}}{2}<\phi(\mb w)+\frac{\theta_{\mc C}}{2}\leq\frac{\pi}{2}\,.
\end{equation}
Then $\phi(\mb w)$ is constrained by
\begin{equation}
\label{eq:interval:phiw:subcase1}
0\leq\phi(\mb w)\leq{\pi}/{2}-{\theta_{\mc C}}/{2}\,.
\end{equation}
Furthermore, the integral in \eqref{eq:ojbective_cos} is rewritten as 
\begin{align}
\int_{\phi(\mb w)-\theta_{\mc C}/2}^{\phi(\mb w)+\theta_{\mc C}/2} |\cos\theta| d\theta
&=\int_{\phi(\mb w)-\theta_{\mc C}/2}^{\phi(\mb w)+\theta_{\mc C}/2} \cos\theta d\theta \nonumber \\
&=\sin\left(\phi(\mb w)+\frac{\theta_{\mc C}}{2}\right)-\sin\left(\phi(\mb w)-\frac{\theta_{\mc C}}{2}\right)\nonumber\\&
=2\cos\left(\phi(\mb w)\right)\sin\left({\frac{\theta_{\mc C}}{2}}\right)\label{eq:subresult_1}\,.
\end{align}
Since $\sin(\theta_{\mc C}/2) \geq 0$, the expression in \eqref{eq:subresult_1} monotonically decreases in $\phi(\mb w)$ for the interval given in \eqref{eq:interval:phiw:subcase1}.
Thus the maximum (resp. minimum) is attained as $2\sin(\theta_{\mc C}/2)$ at $\phi(\mb w) = 0$ (resp. $2\sin^2(\theta_{\mc C}/2)$ at $\phi(\mb w) = \pi/2 - \theta_{\mc C}/2$).\\

\noindent\textbf{Case 2:} Suppose that
\begin{equation}
\label{eq:case2}
- \frac{\pi}{2} \leq \phi(\mb w)-\frac{\theta_{\mc C}}{2} < \frac{\pi}{2} < \phi(\mb w)+\frac{\theta_{\mc C}}{2}\leq\frac{3\pi}{2}\,.
\end{equation} 
Then $\phi(\mb w)$ satisfies
\begin{equation}
\label{eq:case2_phiw}
\frac{\pi}{2}-\frac{\theta_{\mc C}}{2} \leq \phi(\mb w) \leq \frac{\pi}{2}+\frac{\theta_{\mc C}}{2}\,
\end{equation} 
and the integral in \eqref{eq:ojbective_cos} reduces to
\begin{align}
\int_{\phi(\mb w)-\theta_{\mc C}/2}^{\phi(\mb w)+\theta_{\mc C}/2}& |\cos\theta| d\theta\nonumber\\
&=\int_{\phi(\mb w)-\frac{\theta_{\mc C}}{2}}^{\frac{\pi}{2}} \cos\theta d\theta-\int_{\frac{\pi}{2}}^{\phi(\mb w)+\frac{\theta_{\mc C}}{2}} \cos\theta d\theta\nonumber\\&=2-\sin\left(\phi(\mb w)-\frac{\theta_{\mc C}}{2}\right)-\sin\left(\phi(\mb w)+\frac{\theta_{\mc C}}{2}\right)\nonumber\\&=2-2\sin\left(\phi(\mb w)\right)\cos\left({\frac{\theta_{\mc C}}{2}}\right)\label{eq:obj2}\,.
\end{align} 
Since $\cos(\theta_{\mc C}/2) \geq 0$ for all $\theta_{\mc C} \in [0, \pi]$, the maximum (resp. minimum) is attained as $2\sin^2({\theta_{\mc C}}/{2})$ at $\phi(\mb w) = \pi/2 - \theta_{\mc C}/2$ (resp. $4\sin^2({{\theta_{\mc C}}/{4}})$ at $\phi(\mb w) = \pi/2$). \\

\noindent\textbf{Case 3:} Suppose that
\begin{equation}
\label{eq:case3}
\frac{\pi}{2}\leq{\phi(\mb w)-\frac{\theta_{\mc C}}{2}}<\phi(\mb w)+\frac{\theta_{\mc C}}{2}\leq\frac{3\pi}{2}\,.
\end{equation}
Then we have
\begin{equation}
\label{eq:case3:phiw}
\frac{\pi}{2}+\frac{\theta_{\mc C}}{2} \leq \phi(\mb w) \leq \pi\,
\end{equation}
and 
\begin{align}
\int_{\phi(\mb w)-\theta_{\mc C}/2}^{\phi(\mb w)+\theta_{\mc C}/2}& |\cos\theta| d\theta\nonumber\\
&=\int_{\phi(\mb w)-\theta_{\mc C}/2}^{\phi(\mb w)+\theta_{\mc C}/2} (-\cos\theta) d\theta \nonumber \\
&=\sin\left(\phi(\mb w)-\frac{\theta_{\mc C}}{2}\right)-\sin\left(\phi(\mb w)+\frac{\theta_{\mc C}}{2}\right)\nonumber\\
&=-2\cos{\phi(\mb w)}\sin{\frac{\theta_{\mc C}}{2}}\,\label{eq:obj3}.
\end{align} 
The maximum (resp. minimum) of \eqref{eq:obj3} is attained as $2\sin({{\theta_{\mc C}}/{2}})$ at $\phi(\mb w)=\pi$ (resp. $2\sin^2({\theta_{\mc C}}/{2})$ at $\phi(\mb w)={\pi}/{2}+{\theta_{\mc C}}/{2}$).

By combining the results in the above three cases, we obtain \eqref{eq2:proof:lem:Expect} and \eqref{eq2a:proof:lem:Expect}. 
Then substituting the expectation term in \eqref{eq1:proof:lem:Expect} by \eqref{eq2:proof:lem:Expect} provides the first assertion. 

Next we prove the second assertion. 
Similarly to \eqref{eq1:proof:lem:Expect}, we have
\begin{align*}
&\sup_{\begin{subarray}{l} \mb w \in \mbb S^{1} \end{subarray}} \E \bbone_{{\mc C}_{j}}(\mb g) \langle \mb g, \mb w \rangle_+
\\
&=\sup_{\begin{subarray}{l} \mb w \in \mbb S^{1} \end{subarray}} \, P\left\{\mb g\in {\mc C}\right\} \E \left[\|\mb g\|_2 \cdot \left\langle \frac{\mb g}{\|\mb g\|_2}, \mb w \right\rangle_+\,\bigg|\,\frac{\mb g}{\|\mb g\|_2}\in {\mc C}\right]\\
&\overset{(a)}{=}\sup_{\begin{subarray}{l}  \mb w \in \mbb S^{1} \end{subarray}} \, P\left\{\mb g\in {\mc C}\right\} \E \left[{\|\mb g\|_2}\right]\E \left[  \left\langle \frac{\mb g}{\|\mb g\|_2}, \mb w \right\rangle_+\,\bigg|\,\frac{\mb g}{\|\mb g\|_2}\in {\mc C}\right]\\
&\overset{(b)}{=}\frac{\sqrt{2}\Gamma(3/2)}{\Gamma(2)}\sup_{\begin{subarray}{l} \mb w \in \mbb S^{1} \end{subarray}} \, \frac{\theta_{\mc C}}{2\pi} \E \left[  \left\langle \frac{\mb g}{\|\mb g\|_2}, \mb w \right\rangle_+\,\bigg|\,\frac{\mb g}{\|\mb g\|_2}\in {\mc C}\right]\,,
\end{align*} 
where (a) holds since $\|\mb g\|_2$ and $\mb g/\|\mb g\|_2$ are independent, (b) follows from $\E \norm{\mb g}_2 = \sqrt{2}\Gamma(3/2)/\Gamma(2)$, and \[
\P\{\mb g\in\mc C\}=\P\left\{\frac{\mb g}{\|\mb g\|_2}\in\mc C\right\}=\frac{\theta_{\mc C}}{2\pi}\,.
\]
If suffices to show that 
\begin{equation}
\label{eq2:proof:lem:Expect2}
\max_{\mb w \in \mbb S^1} 
\E \left[  \left\langle \frac{\mb g}{\|\mb g\|_2}, \mb w \right\rangle_+ \,\left|\, \frac{\mb g}{\|\mb g\|_2}\in {\mc C} \right. \right] 
=
\frac{2}{\theta_{\mc C}}\sin\left({\frac{\theta_{\mc C}}{2}}\right)\,.
\end{equation}
Since ${\mb g}/{\|\mb g\|_2}$ is uniformly distributed on the unit sphere $\mbb S^1$ and $u_+ = (u + |u|)/2$ for all $u \in \mbb R$, we have
\begin{align}
& \E \left[ \left\langle \frac{\mb g}{\|\mb g\|_2}, \mb w \right\rangle_+ \bigg|\frac{\mb g}{\|\mb g\|_2}\in {\mc C}\right] \nonumber \\
&= \int_{\phi(\mb w)-\theta_{\mc C}/2}^{\phi(\mb w)+\theta_{\mc C}/2} \frac{\cos\theta+|\cos\theta|}{2 \theta_{\mc C}} d\theta\nonumber\\
&=\frac{1}{2}\left(\frac{1}{\theta_{\mc C}}\int_{\phi(\mb w)-\theta_{\mc C}/2}^{\phi(\mb w)+\theta_{\mc C}/2}|\cos\theta| d\theta+\frac{1}{\theta_{\mc C}}\int_{\phi(\mb w)-\theta_{\mc C}/2}^{\phi(\mb w)+\theta_{\mc C}/2}\cos\theta d\theta\right)\,\nonumber \\
&=\frac{1}{2}\bigg(\E \left[ \left|\left\langle \frac{\mb g}{\|\mb g\|_2}, \mb w \right\rangle\right|\,\bigg|\,\frac{\mb g}{\|\mb g\|_2}\in {\mc C}\right]\nonumber\\
&~~~~~~~~~~~~~~~~~~~~~~~~~+\E \left[ \left\langle \frac{\mb g}{\|\mb g\|_2}, \mb w \right\rangle \,\bigg|\,\frac{\mb g}{\|\mb g\|_2}\in {\mc C}\right]\,\bigg).\label{eq:objective_twoexpect}
\end{align}
As shown above, the first term in \eqref{eq:objective_twoexpect} is maximized at $\phi(\mb w) = 0$ and the maximum is given in \eqref{eq2a:proof:lem:Expect}. 
Furthermore, the second term in \eqref{eq:objective_twoexpect} is rewritten as
\begin{align}
\int_{\phi(\mb w)-\theta_{\mc C}/2}^{\phi(\mb w)+\theta_{\mc C}/2}\cos\theta d\theta
&=\sin\left(\phi(\mb w)+\frac{\theta_{\mc C}}{2}\right)-\sin\left(\phi(\mb w)-\frac{\theta_{\mc C}}{2}\right) \nonumber \\
&=2\cos{\phi(\mb w)}\sin\left(\frac{\theta_{\mc C}}{2}\right)\,. \label{eq:obj4}
\end{align}
Since $\sin\left(\theta_{\mc C}/2\right)\geq0$, the expression in \eqref{eq:obj4} is a decreasing function of $\phi(\mb w)\in[0,\pi]$.
Hence, the maximum is attained at $\phi(\mb w)=0$ as 
\begin{equation}
\label{eq:max_7}
\max_{\mb w\in\mbb{S}^1}{2\cos\phi(\mb w)\sin\left(\frac{\theta_{\mc C}}{2}\right)}=2\sin\left(\frac{\theta_{\mc C}}{2}\right)\,.
\end{equation}
Since the two terms in \eqref{eq:objective_twoexpect} are maximized simultaneously, by plugging in the above results to \eqref{eq2:proof:lem:Expect2}, the second assertion is obtained.

{
\subsection{Proof of Lemma~\ref{lem:noise_norm_bound}}
\label{sec:proof:noise_norm_bound}
By construction, we have
\[
\frac{1}{n} \sum_{i=1}^N w_i \bm x_i \sim \mathrm{Normal}\left(\bm 0, \frac{\norm{\bm w}_2^2}{n^2} \bm I_p\right).
\]
Then, the concentration of the Euclidean norm of a standard Gaussian vector guarantees, with probability at least $1-\delta/2$, that
{
\begin{equation}
\label{eq:argbnd1}
\left\|\frac{1}{n}\sum_{i=1}^n w_i\mb x_i\right\|_2
\lesssim
\frac{\|\mb w\|_2}{{n}}(\sqrt{p}+\sqrt{\log(1/\delta)})
\end{equation}
for some absolute constant $C$. This implies the first bound in \eqref{eq:bounds_noise_terms1}.}

Next, we want to obtain an upper bound on the second term in \eqref{eq:bounds_noise_terms1}. By the variational characterization of the spectral norm, we have
\begin{equation}
\label{eq:bnd_tau_i_2}
\begin{aligned}
\left\|\frac{1}{n}\sum_{i=1}^{n}w_i\left(\mb x_i\mb x_i^\T-\mb I_p\right)\right\|  \leq\sup_{\mb u\in{\mathbb B}_2^{p}}\left|\frac{1}{n}\sum_{i=1}^{n}w_i\left((\mb x_i^\T\mb u)^2-1\right)\right|.
\end{aligned}
\end{equation}
For brevity, we introduce a shorthand notation to denote the following random process 
\[
Y_{\mb u}:=\sum_{i=1}^{n}w_i\left((\mb x_i^\T\mb u)^2-1\right),
\]
indexed by $\mb u\in{\mathbb B}_2^p$. Then, for $\mb u,\mb u'\in\mathbb{B}_2^p,$ we have
\[
Y_{\mb u}-Y_{\mb u}=\sum_{i=1}^n w_i\langle\mb x_i,\mb u-\mb u'\rangle\langle\mb x_i,\mb u+\mb u'\rangle.
\]
Therefore, we bound the subexponential norm of each summand as 
\begin{align*}
&\left\|w_i\langle\mb x_i,\mb u-\mb u'\rangle\langle\mb x_i,\mb u+\mb u'\rangle\right\|_{\psi_{1}}\\
&\leq w_i\|\langle\mb x_i,\mb u-\mb u'\rangle\|_{\psi_2}\cdot\|\mb x_i,\mb u+\mb u'\|_{\psi_2}\lesssim w_i\|\mb u-\mb u'\|_2.
\end{align*}
Applying the Bernstein inequality (e.g. see \cite[Theorem~2.8.1]{vershynin2018high}) then yields 
\begin{equation}
\label{eq:Bernstein}
\begin{aligned}
&\P\left(\left|Y_{\mb u}-Y_{\mb u'}\right|\geq c\left(\sqrt{t}\|\mb w\|_2\|\mb u-\mb u'\|_2+t\|\mb w\|_{\infty}\|\mb u-\mb u'\|_2\right)\right)\\
&\qquad\leq2\exp(-t),
\end{aligned}
\end{equation} for any $t\geq0$ and an absolute constant $c$. Then, the process $Y_{\mb u}$ has mixed tail increments (i.e, see \cite[Equation~12]{dirksen2015tail}) with respect to the metrics $(d_1,d_2)$ where $d_1(\mb a,\mb b)=\|\mb w\|_{\infty}\|\mb a-\mb b\|_2$ and $d_2(a,b)=\|\mb w\|_2\|\mb a-\mb b\|_2$ for any $\mb a,\mb b\in\mathbb{B}_2^p$. Hence, applying \cite[Corollary~5.2]{dirksen2015tail} with the bound on $\gamma$-functional (i.e, see \cite[Equation~4]{dirksen2015tail}) provides
\begin{equation*}
\begin{aligned}
&\sup_{\mb u\in\mathbb{B}_2^p}|Y_{\mb u}| \\ 
&\lesssim{\|\mb w\|_2}\left(\int_{0}^{\infty}\sqrt{\log{N\left({\mathbb B}_2^p,\|\cdot\|_2,\eta\right)}} d\eta+\sqrt{\log(1/\delta)}\right) \\ 
&\qquad  +\|\mb w\|_{\infty}\left(\int_{0}^{\infty}{\log{N\left({\mathbb B}_2^p,\|\cdot\|_2,\eta\right)}} d\eta+\log(1/\delta)\right)\\
&\overset{\mathrm{(b)}}{\leq}\|\mb w\|_2(\sqrt{p}+\sqrt{\log(1/\delta)})+\|\mb w\|_{\infty}(p+\log(1/\delta)),
\end{aligned}
\end{equation*} holds with probability at least $1-\delta/2$ where $\mathrm{(b)}$ holds due to an upper bound on the covering number $N({\mathbb B}_2^p,\|\cdot\|_2,\eta)\leq(3/\eta)^p$ (e.g. see \cite[Example~8.1.11]{vershynin2018high}). 
{ This implies the second bound in \eqref{eq:bounds_noise_terms1}.}

}

\subsection{Proof of Lemma~\ref{lem:rhs}}
\label{sec:proof:lem:rhs}

For any $\mb z$ satisfying $\norm{\mb z}_{1,2} = 1$, we have
\begin{equation}
\label{eq:lb_Vz}
V_{\mb z} \geq \min_{\norm{\mb z}_{1,2} = 1} \E V_{\mb z} - \sup_{\mb z \in B_{1,2}} |V_{\mb z} - \E V_{\mb z}|\,.
\end{equation}
In what follows, we derive lower estimates of the summands in the right-hand side of \eqref{eq:lb_Vz}. 

First, we derive a lower bound on $\min_{\norm{\mb z}_{1,2} = 1} \E V_{\mb z}$. 
Since $\mb x_1,\dots,\mb x_n$ are i.i.d. $\mathrm{Normal}(\mb 0, \mb I_p)$, we have
\begin{align*}
\E V_{\mb z}
&= \E \, \frac{1}{n} \sum_{i=1}^n \sum_{j=1}^k \bbone_{\mc C_j}(\mb x_i) \left| \langle \mb x_i, \mb z_j \rangle \right|
= \E \, \sum_{j=1}^k \bbone_{\mc C_j}(\mb g) \left| \langle \mb g, \mb z_j \rangle \right|\\
&= \sum_{j=1}^k \norm{\mb z_j}_2 \E \, \bbone_{\mc C_j}(\mb g) \left| \left\langle \mb g, \frac{\mb z_j}{\norm{\mb z_j}_2} \right\rangle \right|\,,
\end{align*}
where $\mb z = [\mb z_1;\ \dots ;\ \mb z_k]$. 
Then $\E V_{\mb z}$ is lower-bounded by
\[
\E V_{\mb z}
\geq \norm{\mb z}_{1,2} \inf_{j \in [k], \mb w \in \mbb S^{p-1}} \E \, \bbone_{\mc C_j}(\mb g) \left| \left\langle \mb g, \mb w \right\rangle \right|\,.
\]

Next, we show that $(V_{\mb z} - \E V_{\mb z})_{\mb z \in B_{1,2}}$ is concentrated around $0$ with high probability by using the following lemma. 

\begin{lem}
\label{lem:concentration_Uz}
Suppose that $\mc A_1, \dots, \mc A_k$ be disjoint subsets in $\mathbb{R}^p$. 
Let $(U_{\mb z})_{\mb z \in B_{1,2}}$ be a random process defined by 
\begin{equation}
\label{eq:def_Uz}
U_{\mb z} := \frac{1}{n} \sum_{i=1}^n \sum_{j=1}^k \bbone_{\mc A_j}(\mb x_i) \langle \mb x_i, \mb z_j \rangle_+\,,
\end{equation}
where $\mb x_1,\dots,\mb x_n$ are i.i.d. $\mathrm{Normal}(\mb 0,\mb I_p)$. Then, for any $\delta\in (0,1)$, there exists an absolute constant $c > 0$ such that
\begin{equation}
\label{eq:tailbound_Uz}
\sup_{\mb z \in B_{1,2}} \left| U_{\mb z} - \E U_{\mb z}\right|  
\leq 
c \left(\frac{p \log^3p \log^5k + \log(\delta^{-1}) \log k}{n}\right)^{1/2}\,
\end{equation}
holds with probability at least $1-\delta$. 
\end{lem}

\begin{IEEEproof}
We first show that $U_{\mb z}$ has sub-Gaussian increments with respect to the $\ell_\infty^k(\ell_2^p)$-norm, i.e.
\begin{equation}
\label{eq:subginc}
\norm{U_{\mb z} - U_{\mb z'}}_{\psi_2}\lesssim \frac{\sqrt{\log k}}{\sqrt{n}} \, \norm{\left(\mb z_j\right)_{j=1}^k - \left(\mb z'_j\right)_{j=1}^k}_{\ell_\infty^k(\ell_2^p)}\,.
\end{equation}
Since $\mc A_{1},\dots,\mc A_{k}$ are disjoint, it follows that 
\begin{align}
\left|{U_{\mb z} - U_{\mb z'}}\right|
&\leq
\frac{1}{n}\sum_{i=1}^{n}\sum_{j=1}^{k}\bbone_{\mc A_j}(\mb x_i) \left|\langle\mb x_i,\mb z_j-\mb z_j'\rangle\right|\nonumber\\
&\leq\frac{1}{n}\sum_{i=1}^{n}\max_{1\leq{j}\leq{k}}\left|\langle\mb x_i,\mb z_j-\mb z_j'\rangle\right|\, \label{eq:ubVzabs}\,
\end{align}
holds almost surely, where the last step follows from H\"older's inequality. We proceed with the following lemma.

\begin{lem}[{\cite[Lemma~2.2.2]{van1996weak}}]
\label{lem:subgnormax}
Let $\mb g \sim \mathrm{Normal}(\mb 0, \mb I_p)$ and $\mb a_1,\dots,\mb a_k \in \mbb{R}^p$. Then
\begin{equation*}
\norm{\max_{j \in [k]} \left| \langle \mb g, \mb a_j \rangle \right|}_{\psi_2}
\lesssim \sqrt{\log k} \max_{j \in [k]} \norm{\mb a_j}_2\,.
\end{equation*}
\end{lem}

It follows from \eqref{eq:ubVzabs} and Lemma~\ref{lem:subgnormax} that
\begin{align*}
\norm{U_{\mb z} - U_{\mb z'}}_{\psi_2}
&\leq \norm{\frac{1}{n}\sum_{i=1}^{n}\max_{j \in [k]}\left|{\langle{\mb x_i,\mb z_j-\mb z_j'}\rangle}\right|}_{\psi_2} \\
&\lesssim \frac{1}{n}\sqrt{\sum_{i=1}^{n} \norm{\max_{j \in [k]} \left|{\langle{\mb x_i,\mb z_j-\mb z_j'}\rangle}\right|}_{\psi_2}^2}\\
& \lesssim \frac{\sqrt{\log k}}{\sqrt{n}} \max_{j \in [k]} \norm{\mb z_j - \mb z'_j}_2 \\
& = \frac{\sqrt{\log k}}{\sqrt{n}} \, \norm{\left(\mb z_j\right)_{j=1}^k - \left(\mb z'_j\right)_{j=1}^k}_{\ell_\infty^k(\ell_2^p)}\,,
\end{align*}
where the second inequality follows from \cite[Proposition~2.6.1]{vershynin2018high}. 

Since $U_z$ has a sub-Gaussian increment as in \eqref{eq:subginc}, by \cite[Lemma~2.6.8]{vershynin2018high}, which says that centering does not harm the sub-gaussianity, we also have 
\begin{align}
&\norm{\left(U_{\mb z} - \E U_{\mb z}\right) - \left(U_{\mb z'} - \E U_{\mb z'}\right)}_{\psi_2}\nonumber\\
&\qquad\lesssim \frac{\sqrt{\log k}}{\sqrt{n}} \, \norm{\left(\mb z_j\right)_{j=1}^k - \left(\mb z'_j\right)_{j=1}^k}_{\ell_\infty^k(\ell_2^p)}\,.
\label{eq:subginc2}
\end{align} 
Therefore Dudley's inequality \cite{dudley1967sizes} applies to provide a tail bound on the left-hand side of \eqref{eq:tailbound_Uz}. 
Specifically it follows from a version of Dudley's inequality \cite[Theorem~8.1.6]{vershynin2018high} that
\begin{align}
&\sup_{\mb z \in B_{1,2}} \left| U_{\mb z} - \E U_{\mb z} \right|\lesssim\nonumber\\
& \frac{\sqrt{\log{k}}}{\sqrt{n}}\left(\int_0^\infty \sqrt{\log N(B_{1,2}, \norm{\cdot}_{\ell_\infty^k(\ell_2^p)},\eta)} d\eta+u \, \mathrm{diam}\left( B_{1,2}\right)\right)\, 
\label{eq:ubbydudley}
\end{align} 
holds with probability at least $1-2\exp(-u^2)$. 
Note that the diameter term in \eqref{eq:ubbydudley} is trivially upper-bounded by 
\begin{align*}
\mbox{diam}( B_{1,2}) 
= \sup_{\mb z,\mb z' \in  B_{1,2}} \norm{\mb z - \mb z'}_{\ell_\infty^k(\ell_2^p)} \leq 2\,.
\end{align*}
Moreover, since $B_{1,2} \subseteq \sqrt{p} B_1$, where $B_1$ denotes the unit ball in $\ell_1$, we have
\begin{align*}
\int_0^\infty& \sqrt{\log N(  B_{1,2}, \norm{\cdot}_{\ell_\infty^k(\ell_2^p)},\eta)} d\eta\\
& \leq \int_0^\infty \sqrt{\log N( \sqrt{p} B_1, \norm{\cdot}_{\ell_\infty^k(\ell_2^p)},\eta)} d\eta \\
& \lesssim  \sqrt{p} \log^{3/2} p \log^2 k\,,
\end{align*}
where the second inequality follows from Maurey's empirical method \cite{carl1985inequalities} (also see \cite[Lemma~3.4]{junge2020generalized}).
By plugging in these estimates to \eqref{eq:ubbydudley}, we obtain that
\[
\sup_{\mb z \in  B_{1,2}} \left| U_{\mb z} - \E U_{\mb z} \right| \lesssim \left(\frac{p \log^3p \log^5k + \log(\delta^{-1}) \log k}{n}\right)^{1/2}\,
\] 
holds with probability at least $1-\delta$.
\end{IEEEproof}

Note that $\mc C_1,\dots, \mc C_k$ are disjoint except on a boundary, which corresponds to a set of measure zero. 
Since the standard multivariate normal distribution is absolutely continuous relative to the Lebesgue measure, these null sets can be ignored in getting a tail bound on the infimum of the random process $(V_{\mb z})_{\mb z \in B_{1,2}}$.
Moreover, $V_{\mb z}$ is written as $V_{\mb z} = V_{\mb z}^+ + V_{\mb z}^-$, where
\[
V_{\mb z}^+ := \frac{1}{n} \sum_{i=1}^n \sum_{j=1}^k \bbone_{\mc C_j}(\mb x_i) \langle \mb x_i, \mb z_j \rangle_+\,
\]
and
\[
V_{\mb z}^- := \frac{1}{n} \sum_{i=1}^n \sum_{j=1}^k \bbone_{\mc C_j}(\mb x_i) \langle \mb x_i, - \mb z_j \rangle_+\,.
\]
Since $(V_{\mb z}^+)_{\mb z \in B_{1,2}}$ and $(V_{\mb z}^-)_{\mb z \in B_{1,2}}$ are in the form of \eqref{eq:def_Uz}, by Lemma~\ref{lem:concentration_Uz}, we obtain that
\begin{align}
\sup_{\mb z \in  B_{1,2}} \left| V_{\mb z} - \E V_{\mb z} \right| 
& \leq \sup_{\mb z \in  B_{1,2}} \left| V_{\mb z}^+ - \E V_{\mb z}^+ \right| 
+ \sup_{\mb z \in  B_{1,2}} \left| V_{\mb z}^- - \E V_{\mb z}^- \right| \nonumber \\
& \lesssim \left(\frac{p \log^3p \log^5k + \log(\delta^{-1}) \log k}{n}\right)^{1/2} \label{eq:concentration_Vz}\,
\end{align} 
holds with probability at least $1-\delta/2$.

Finally, the assertion is obtained by plugging in the above estimates to \eqref{eq:lb_Vz}.

\subsection{Proof of Lemma~\ref{lem:lhs}}
\label{sec:proof:lem:lhs}

Note that $Q_{\mb z}$ is decomposed into   
\begin{align}
&Q_{\mb z}= \frac{1}{n} \sum_{i=1}^n \sum_{j=1}^k 
\bbone_{\tilde{\mc C}_{j}\setminus \mc C_j}(\mb x_i)
\langle \mb x_i, \mb z_j \rangle\nonumber\\
&~~~~~~~~~~~~~~~~~~~~~~~+ \frac{1}{n} \sum_{i=1}^n \sum_{j=1}^k 
\bbone_{\mc C_j \setminus \tilde{\mc C}_{j}}(\mb x_i) \langle \mb x_i, - \mb z_j \rangle\,.\label{eq:decompQz}
\end{align}
Then the summands in the right-hand side of \eqref{eq:decompQz} are respectively upper-bounded by
\begin{equation*}
Q'_{\mb z} := \frac{1}{n} \sum_{i=1}^n \sum_{j=1}^k 
\bbone_{\tilde{\mc C}_{j}\setminus \mc C_j}(\mb x_i) \langle \mb x_i, \mb z_j \rangle_+\,
\end{equation*}
and
\begin{equation*}
Q''_{\mb z} := \frac{1}{n} \sum_{i=1}^n \sum_{j=1}^k 
\bbone_{\mc C_j \setminus \tilde{\mc C}_{j}}(\mb x_i) \langle \mb x_i, -\mb z_j \rangle_+\,.
\end{equation*}
We upper-bound $\sup_{\mb z \in  B_{1,2}} Q'_{\mb z}$ and $\sup_{\mb z \in  B_{1,2}} Q''_{\mb z}$ to get an upper bound on $\sup_{\mb z \in  B_{1,2}} Q_{\mb z}$ through \eqref{eq:decompQz} by the triangle inequality. 
Specifically, we show that there exists an absolute constant $c > 0$ such that
\begin{align}
\sup_{\norm{\mb z}_{1,2}=1} Q'_{\mb z}&\leq 
\sup_{j \in [k], \mb w \in \mbb S^{p-1}} \E \bbone_{\tilde{\mc C}_{j}\setminus \mc C_j}(\mb g) \langle \mb g, \mb w \rangle_+
\nonumber\\
&+ c\left(\frac{p \log^3p \log^5k + \log(\delta^{-1}) \log k}{n}\right)^{1/2}\,
\label{eq:upb1statmnt1}
\end{align}
and
\begin{align*}
\sup_{\norm{\mb z}_{1,2} = 1} Q''_{\mb z} 
&\leq 
\sup_{j \in [k], \mb w \in \mbb S^{p-1}} \E \bbone_{\mc C_j \setminus \tilde{\mc C}_{j}}(\mb g) \langle \mb g, \mb w \rangle_+ \nonumber\\
&+ c\left(\frac{p \log^3p \log^5k + \log(\delta^{-1}) \log k}{n}\right)^{1/2}\,
\end{align*}
hold simultaneously with probability at least $1-\delta/2$.

Due to the symmetry, it suffices to show that \eqref{eq:upb1statmnt1} holds with probability $1-\delta/4$. 
By the triangle inequality, it follows that
\[
\sup_{\norm{\mb z}_{1,2} = 1} Q'_{\mb z} 
\leq \sup_{\norm{\mb z}_{1,2}=1} \E Q'_{\mb z} + \sup_{\mb z \in B_{1,2}} |Q'_{\mb z} - \E Q'_{\mb z}|\,.
\]
Then, similar to Lemma~\ref{lem:rhs}, we derive \eqref{eq:upb1statmnt1} through the concentration of the maximum deviation, that is, $\sup_{\mb z \in  B_{1,2}} |Q'_{\mb z} - \E Q'_{\mb z}|$, and an upper bound on $\sup_{\mb z \in  B_{1,2}} \E Q'_{\mb z}$. 
The supremum of the expectation is upper-bounded as
\begin{align*}
\E Q'_{\mb z} &= \E \sum_{j=1}^k \bbone_{\tilde{\mc C}_{j}\setminus \mc C_j}(\mb g) \langle \mb g, \mb z_j \rangle_+ \\
&\leq \max_{j \in [k], \mb w \in \mbb S^{p-1}} \E \bbone_{\tilde{\mc C}_{j}\setminus \mc C_j}(\mb g) \langle \mb g, \mb w \rangle_+ \sum_{j=1}^k \norm{\mb z_j}_2\,.
\end{align*}

Moreover, since $\tilde{\mc C}_1, \dots, \tilde{\mc C}_k$ are disjoint (except on a set of measure zero), by Lemma~\ref{lem:concentration_Uz}, we obtain that 
\begin{equation}
\label{eq:concentration2}
\sup_{\mb z \in  B_{1,2}} \left| Q'_{\mb z} - \E Q'_{\mb z} \right| 
\lesssim \left(\frac{p \log^3p \log^5k + \log(\delta^{-1}) \log k}{n}\right)^{1/2}\,
\end{equation} 
holds with probability at least $1-\delta/4$.
This provides the assertion in \eqref{eq:upb1statmnt1}.

\section*{Acknowledgements}
{The authors appreciate anonymous reviewers for their insightful feedback and constructive suggestions that helped significantly improve the technical contributions and presentation of the manuscript.} 
S.K. and K.L. were supported in part by NSF CAREER award CCF 19-43201. S.B. was supported in part by Semiconductor Research Corporation (SRC) and DARPA.
\bibliographystyle{IEEEtran}
\bibliography{ref}

\begin{thebibliography}{10}
\providecommand{\url}[1]{#1}
\csname url@samestyle\endcsname
\providecommand{\newblock}{\relax}
\providecommand{\bibinfo}[2]{#2}
\providecommand{\BIBentrySTDinterwordspacing}{\spaceskip=0pt\relax}
\providecommand{\BIBentryALTinterwordstretchfactor}{4}
\providecommand{\BIBentryALTinterwordspacing}{\spaceskip=\fontdimen2\font plus
\BIBentryALTinterwordstretchfactor\fontdimen3\font minus \fontdimen4\font\relax}
\providecommand{\BIBforeignlanguage}[2]{{%
\expandafter\ifx\csname l@#1\endcsname\relax
\typeout{** WARNING: IEEEtran.bst: No hyphenation pattern has been}%
\typeout{** loaded for the language `#1'. Using the pattern for}%
\typeout{** the default language instead.}%
\else
\language=\csname l@#1\endcsname
\fi
#2}}
\providecommand{\BIBdecl}{\relax}
\BIBdecl

\bibitem{ghosh2021max}
A.~Ghosh, A.~Pananjady, A.~Guntuboyina, and K.~Ramchandran, ``Max-affine regression: Parameter estimation for gaussian designs,'' \emph{IEEE Transactions on Information Theory}, vol.~68, no.~3, pp. 1851--1885, 2021.

\bibitem{bahmani2017phase}
S.~Bahmani and J.~Romberg, ``Phase retrieval meets statistical learning theory: A flexible convex relaxation,'' in \emph{Artificial Intelligence and Statistics}, 2017, pp. 252--260.

\bibitem{bahmani2019estimation}
S.~Bahmani, ``Estimation from nonlinear observations via convex programming with application to bilinear regression,'' \emph{Electronic Journal of Statistics}, vol.~13, no.~1, pp. 1978--2011, 2019.

\bibitem{bahmani2019solving}
S.~Bahmani and J.~Romberg, ``Solving equations of random convex functions via anchored regression,'' \emph{Foundations of Computational Mathematics}, vol.~19, no.~4, pp. 813--841, 2019.

\bibitem{goldstein2018phasemax}
T.~Goldstein and C.~Studer, ``Phasemax: Convex phase retrieval via basis pursuit,'' \emph{IEEE Transactions on Information Theory}, vol.~64, no.~4, pp. 2675--2689, 2018.

\bibitem{candes2013phaselift}
E.~J. Candes, T.~Strohmer, and V.~Voroninski, ``Phaselift: Exact and stable signal recovery from magnitude measurements via convex programming,'' \emph{Communications on Pure and Applied Mathematics}, vol.~66, no.~8, pp. 1241--1274, 2013.

\bibitem{waldspurger2015phase}
I.~Waldspurger, A.~d’Aspremont, and S.~Mallat, ``Phase recovery, maxcut and complex semidefinite programming,'' \emph{Mathematical Programming}, vol. 149, no. 1-2, pp. 47--81, 2015.

\bibitem{balestriero2018mad}
R.~Balestriero and R.~Baraniuk, ``Mad max: Affine spline insights into deep learning,'' \emph{arXiv preprint arXiv:1805.06576}, 2018.

\bibitem{balestriero2019geometry}
R.~Balestriero, R.~Cosentino, B.~Aazhang, and R.~Baraniuk, ``The geometry of deep networks: Power diagram subdivision,'' in \emph{Advances in Neural Information Processing Systems}, 2019, pp. 15\,806--15\,815.

\bibitem{balestriero2020max}
R.~Balestriero, S.~Paris, and R.~Baraniuk, ``Max-affine spline insights into deep generative networks,'' \emph{arXiv preprint arXiv:2002.11912}, 2020.

\bibitem{siahkamari2019learning}
A.~Siahkamari, V.~Saligrama, D.~Castanon, and B.~Kulis, ``Learning {B}regman divergences,'' \emph{arXiv preprint arXiv:1905.11545}, 2019.

\bibitem{balazs2016convex}
G.~Bal{\'a}zs, ``Convex regression: Theory, practice, and applications,'' Ph.D. dissertation, University of Alberta, 2016.

\bibitem{magnani2009convex}
A.~Magnani and S.~P. Boyd, ``Convex piecewise-linear fitting,'' \emph{Optimization and Engineering}, vol.~10, no.~1, pp. 1--17, 2009.

\bibitem{toriello2012fitting}
A.~Toriello and J.~P. Vielma, ``Fitting piecewise linear continuous functions,'' \emph{European Journal of Operational Research}, vol. 219, no.~1, pp. 86--95, 2012.

\bibitem{hannah2013multivariate}
L.~A. Hannah and D.~B. Dunson, ``Multivariate convex regression with adaptive partitioning,'' \emph{The Journal of Machine Learning Research}, vol.~14, no.~1, pp. 3261--3294, 2013.

\bibitem{ho2019dca}
V.~T. Ho, H.~A. Le~Thi, and T.~P. Dinh, ``{DCA} with successive {DC} decomposition for convex piecewise-linear fitting,'' in \emph{International Conference on Computer Science, Applied Mathematics and Applications}.\hskip 1em plus 0.5em minus 0.4em\relax Springer, 2019, pp. 39--51.

\bibitem{tao1998dc}
P.~D. Tao and L.~T.~H. An, ``A {DC} optimization algorithm for solving the trust-region subproblem,'' \emph{SIAM Journal on Optimization}, vol.~8, no.~2, pp. 476--505, 1998.

\bibitem{candes2010matrix}
E.~J. Candes and Y.~Plan, ``Matrix completion with noise,'' \emph{Proceedings of the IEEE}, vol.~98, no.~6, pp. 925--936, 2010.

\bibitem{van2020deterministic}
J.~van~den Brand, ``A deterministic linear program solver in current matrix multiplication time,'' in \emph{Proceedings of the Fourteenth Annual ACM-SIAM Symposium on Discrete Algorithms}.\hskip 1em plus 0.5em minus 0.4em\relax SIAM, 2020, pp. 259--278.

\bibitem{gurobi}
\BIBentryALTinterwordspacing
L.~Gurobi~Optimization, ``Gurobi optimizer reference manual,'' 2021. [Online]. Available: \url{http://www.gurobi.com}
\BIBentrySTDinterwordspacing

\bibitem{ghosh2019max}
A.~Ghosh, A.~Pananjady, A.~Guntuboyina, and K.~Ramchandran, ``Max-affine regression: Provable, tractable, and near-optimal statistical estimation,'' \emph{arXiv preprint arXiv:1906.09255}, 2019.

\bibitem{diakonikolas2021relu}
I.~Diakonikolas, J.~H. Park, and C.~Tzamos, ``Relu regression with massart noise,'' \emph{Advances in Neural Information Processing Systems}, vol.~34, 2021.

\bibitem{john2008implementation}
E.~John and E.~A. Y{\i}ld{\i}r{\i}m, ``Implementation of warm-start strategies in interior-point methods for linear programming in fixed dimension,'' \emph{Computational Optimization and Applications}, vol.~41, no.~2, pp. 151--183, 2008.

\bibitem{vapnik2015uniform}
V.~N. Vapnik and A.~Y. Chervonenkis, ``On the uniform convergence of relative frequencies of events to their probabilities,'' in \emph{Measures of complexity}.\hskip 1em plus 0.5em minus 0.4em\relax Springer, 2015, pp. 11--30.

\bibitem{tan2019phase}
Y.~S. Tan and R.~Vershynin, ``Phase retrieval via randomized kaczmarz: theoretical guarantees,'' \emph{Information and Inference: A Journal of the IMA}, vol.~8, no.~1, pp. 97--123, 2019.

\bibitem{vershynin2018high}
R.~Vershynin, \emph{High-dimensional probability: An introduction with applications in data science}.\hskip 1em plus 0.5em minus 0.4em\relax Cambridge university press, 2018, vol.~47.

\bibitem{dirksen2015tail}
S.~Dirksen, ``Tail bounds via generic chaining,'' 2015.

\bibitem{van1996weak}
A.~W. van~der Vaart and J.~A. Wellner, \emph{Weak convergence and empirical processes}, ser. Springer Series in Statistics.\hskip 1em plus 0.5em minus 0.4em\relax Springer, 1996.

\bibitem{dudley1967sizes}
R.~M. Dudley, ``The sizes of compact subsets of {H}ilbert space and continuity of {G}aussian processes,'' \emph{Journal of Functional Analysis}, vol.~1, no.~3, pp. 290--330, 1967.

\bibitem{carl1985inequalities}
B.~Carl, ``Inequalities of {B}ernstein-{J}ackson-type and the degree of compactness of operators in {B}anach spaces,'' in \emph{Annales de l'institut Fourier}, vol.~35, no.~3, 1985, pp. 79--118.

\bibitem{junge2020generalized}
M.~Junge and K.~Lee, ``Generalized notions of sparsity and restricted isometry property. {P}art {I}: A unified framework,'' \emph{Information and Inference: A Journal of the IMA}, vol.~9, no.~1, pp. 157--193, 2020.

\end{thebibliography}

%


\begin{IEEEbiographynophoto}{Seonho Kim}
    received the B.S. and M.S. degrees
    in electrical and computer engineering from Ajou
    University, Suwon, South Korea, in 2017 and 2019,
    respectively. He is currently pursuing the Ph.D.
    degree in electrical and computer engineering with
    The Ohio State University, Columbus, OH, USA.
    His main research interests include the areas of
    optimization, machine learning, and data science.
\end{IEEEbiographynophoto}

\begin{IEEEbiographynophoto}{Sohail Bahmani}
    is a researcher broadly interested in the areas of statistics and optimization, and their interplay. He has made contributions in areas such as nonlinear regression, robust estimation, high-dimensional statistics, and inverse problems. He has received a best paper award from the AIStats conference in 2017 for his work on the phase retrieval problem. He was a postdoctoral researcher in the School of Electrical and Computer Engineering at Georgia Tech from 2013 to 2022, and received his PhD in Electrical and Computer Engineering from Carnegie Mellon University in 2013.
\end{IEEEbiographynophoto}

\begin{IEEEbiographynophoto}{Kiryung Lee}
    (Senior Member 2019, IEEE) received the B.S. and the M.S. degrees in electrical engineering from Seoul National University, Seoul, South Korea, in 2000 and 2002, and the Ph.D. in electrical and computer engineering from the University of Illinois at Urbana-Champaign, Urbana, IL, USA. Since 2018, he has been an Assistant Professor with the Electrical and Computer Engineering Department, Ohio State University, Columbus. His research focuses on developing mathematical theory and optimization algorithms for inverse problems in signal processing, imaging, machine learning, statistics, and data science. He is a recipient of the NSF CAREER Award. 
\end{IEEEbiographynophoto}






\end{document}